\documentclass{article}
\usepackage{authblk}

\usepackage[margin=1.6in]{geometry}
\usepackage{graphicx}
\usepackage{latexsym}
\usepackage{subfigure}
\usepackage{float}
\usepackage{amsmath}
\usepackage{algpseudocode}
\usepackage{algorithm}
\usepackage{algorithmicx}
\usepackage{mathtools}
\usepackage{tabularx}
\usepackage{amssymb}
\usepackage{longtable}
\usepackage[english]{babel}
\usepackage{array}
\usepackage{caption}
\usepackage{color}
\usepackage{dcolumn}
\usepackage{bm}
\usepackage{multirow}
\usepackage[table]{xcolor}
\usepackage{textcomp}
\usepackage{hhline}
\usepackage{natbib}
\usepackage{hyperref}
\newcolumntype{L}[1]{>{\raggedright\let\newline\\\arraybackslash\hspace{0pt}}m{#1}}
\newcolumntype{C}[1]{>{\centering\let\newline\\\arraybackslash\hspace{0pt}}m{#1}}
\newcolumntype{R}[1]{>{\raggedleft\let\newline\\\arraybackslash\hspace{0pt}}m{#1}}

\newcommand{\hide}[1]{}

\newtheorem{definition}{Definition}
\newtheorem{theorem}{Theorem}
\newtheorem{corollary}{Corollary}
\newtheorem{lemma}{Lemma}

\newtheorem{example}{Example}
\newtheorem{proof}{proof}

\title{Cluster-based trajectory segmentation with local noise}

\author[1]{Maria Luisa Damiani\footnote{Corresponding author. E-mail: maria.damiani@unimi.it}}
\author[1]{Fatima Hachem}
\author[1]{Issa Hamza}
\author[2,3]{Nathan Ranc}
\author[3]{Paul Moorcroft}
\author[2,3]{Francesca Cagnacci}
\affil[1]{Dept. Computer Science, Universit\'a degli Studi di Milano, Italy}
\affil[2]{Dept. Biodiversity and Molecular Ecology, Fondazione E. Mach, Italy}
\affil[3]{Dept. Organismic and Evolutionary Biology, Harvard University, USA }
\date{}

\begin{document}
\maketitle

\begin{abstract}
	
We present a framework for the partitioning of a spatial trajectory in a sequence of segments based on spatial density and temporal criteria. The result is a set of temporally separated clusters interleaved by sub-sequences of unclustered points. 
A major novelty 
is the proposal of an outlier or \emph{noise} model based on the distinction between intra-cluster (\emph{local noise}) and inter-cluster noise (\emph{transition}): the local noise models the temporary absence from a residence while the transition the definitive departure towards a next residence. 
We  analyze in detail the properties of the model and present a  comprehensive solution for the extraction of temporally ordered clusters. The  effectiveness of the solution is evaluated first qualitatively and next quantitatively by contrasting the segmentation with ground truth. The ground truth consists of a set of trajectories of labeled points simulating animal movement. 
Moreover, we show that the approach can streamline the discovery of additional \emph{derived} patterns, by presenting a novel technique for the analysis of periodic movement.
From a methodological perspective, 
a valuable aspect of this research 
is that it combines the theoretical investigation  with the application and external validation of the segmentation framework. This paves the way to an effective deployment of the solution in broad and challenging fields such as e-science.  
\end{abstract}

\section{Introduction}

Recent years have witnessed a tremendous growth in the collection of trajectory data  
and trajectory data analysis has become a prominent research stream with important applications in e.g. urban computing, intelligent transportation,  animal ecology \citep{Gian2008,Zheng2014,Zheng2015,Parent2013,tsas2015}. 
Spatial trajectories, in particular (simply trajectories hereinafter), are sequences of temporally correlated observations describing the movement of  an object through  a series of points sampling the time-varying location of the object \citep{zheng11}.  An example  is shown in Figure \ref{initt0}. 
\begin{figure}[H] 
	\centering
		\includegraphics[height=4cm]{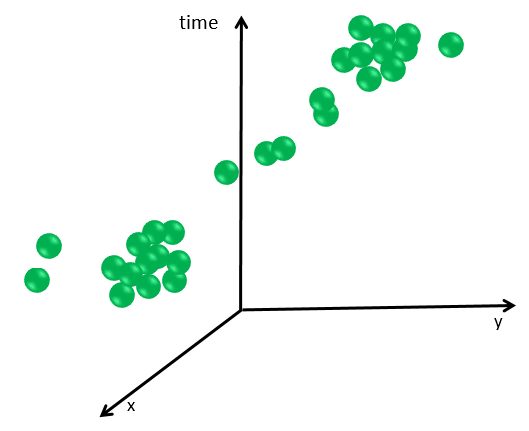}	
	\caption{ A trajectory in the spatio-temporal coordinate system.} 
	\label{initt0}
\end{figure}
For the sake of generality, we do not make any stringent assumption on location sampling frequency and regularity, even on the movement characteristics such as
speed, heading and so forth. The only loose  assumption is that the  temporal distance between two consecutive points is relatively small (with respect to the mobility phenomenon under consideration). 
In this view, the trajectory is simply a sequence of relatively close in time  ordered points. 

A major analysis task over trajectories is trajectory segmentation. Generally speaking, the  segmentation task splits a sequence of data points, in a series of disjoint sub-sequences consisting of points that are homogeneous with respect to some criteria \citep{Keogh2001-a,Aronov2015}.  Diverse segmentation criteria have been proposed in literature, even for different purposes, including time series summarization, e.g. \citep{Keogh2001-a,Esling2012}, and trajectory indexing in databases, e.g. \citep{Rasetic2005,Traj2010}.  In this paper we focus on a slightly different problem, that is splitting  a trajectory in  a-priori unknown number of segments based on spatial density and temporal criteria.   We refer to this problem as cluster-based segmentation.
The cluster-based segmentation problem can be  introduced as follows. Consider a trajectory $T$ of temporally ordered points in a generic metric space. e.g. the Euclidean space, $T=(p_1, t_1),..,(p_n, t_n)$ with $t_i$  a time instant and $p_i$ a point of space, and consider a segmentation for $T$ consisting of a set of temporally ordered sub-sequences, denoted $S_1 < S_2..< S_k$, covering the trajectory, i.e., $T=\bigcup_{i \in[1,k]}{S_i}$.
We say that such segmentation is cluster-based 
if the following conditions hold:
\begin{itemize}
	\item  A subset of segments are clusters representing spatially dense sets of points. 
	Such clusters are thus  \emph{temporally separated}.  
	
	\item The points temporally lying  between two consecutive clusters (or one cluster and the begin/end of the trajectory) 
	are points that cannot be added to any cluster. Such points form a segment called  \emph{transition}.
	
\end{itemize}
The sequence of alternating clusters and transitions represents the trajectory segmentation \emph{induced} by the clustering technique.\\
\\
\textbf{Applications and requirements.} A major application of cluster-based segmentation is to detect \emph{stop-and-move} patterns \citep{Parent2013}. Stop-and-move is an \emph{individual}  pattern \citep{Dodge08}
typically describing the behavior of an object 
that resides in a region of space for some time, i.e. a stop or \emph{residence}, 
and then moves to some other region, in a continuous flow from stop to stop.  This behavior can be exhibited at different temporal scales. For example, Figure \ref{app}.(a) illustrates the trajectory of an animal that during the seasonal migration moves from one \emph{home-range} to another home-range \citep{Damiani:2014}. At a different time scale, Figure \ref{app}.(b) shows 
the wandering of the eye gaze stopping in \emph{fixations}  
during the observation of a scene \citep{Cerf2008}. 
Intuitively, the stop-and-move pattern is exhibited by those phenomena that alternate periods of relative stability to periods of instability.
\begin{figure}[H] 
	\centering
	\subfigure[]
	{\includegraphics[height=4.2cm]{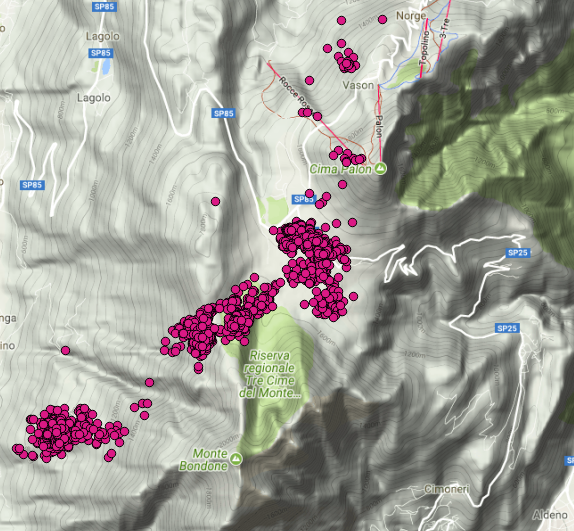}}
	\subfigure[]
	{\includegraphics[height=4.2cm]{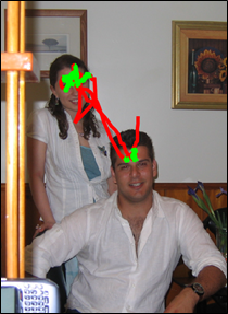}}	
	\caption{Stop-and-move pattern: (a) the seasonal migration of a roe-deer  tracked for over 1 year; (b) the trace of the eye gaze \protect{\citep{Cerf2008}}.} 
	\label{app}
\end{figure}

An interesting question, related to the concept of cluster-based segmentation, 
concerns the meaning of \emph{noise}  in such a setting.   
Broadly speaking, noise consists of data points that  do not fit into the clustering structure and that for such a reason can be considered as diverging 
\citep{Dbscan1996,Han2011}. 
In our case - if we rule out possible errors during the data collection phase - we can see that there are two classes of unclustered points: the points representing a transition and the points 
that do not belong to either a cluster or a transition. In the latter case, the  noise indicates, in essence,  a temporary departure or \emph{absence} from the cluster. 
For example,  this notion of temporary absence can be exemplified by an individual leaving the area where he/she resides, for example for a travel, and then returning back after a period of time. Note that such an absence cannot be qualified as a transition because the individual in reality does not change residence but simply leaves it for a period. 
Put in different terms, absence points represent a form of noise that is somehow local to clusters in contrast to transition, which represents an inter-cluster noise. To emphasize this characteristic, we will refer to this form of noise as \emph{local noise}. Figure \ref{initt0-1} highlights the key difference between  cluster, transition and local noise.
Local noise can have an application-dependent meaning. 
For example, in the field of animal ecology, biologists use the term 'excursion' to characterize the temporary absence from the home-range where the animals reside \citep{Damiani2016}. 
In summary, the local noise  represents a semantically meaningful ingredient of many dynamic phenomena and as such cannot be neglected. \\
\begin{figure}[H] 
	\centering
	\includegraphics[height=4cm]{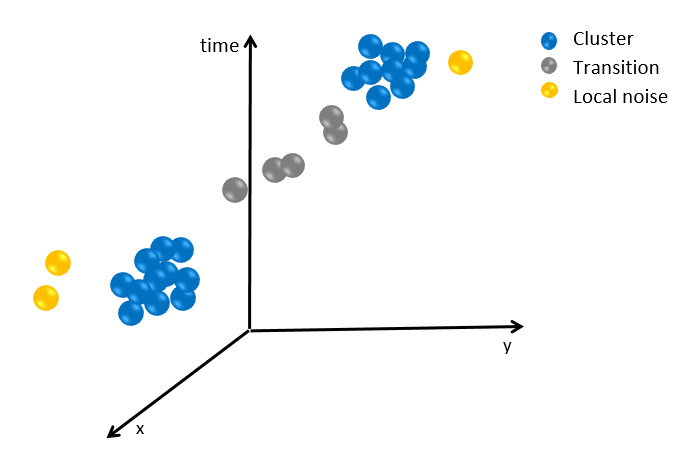}
	\caption{Trajectory segmentation: the trajectory is split into two clusters with a few local noise points, and one transition } 
	\label{initt0-1}
\end{figure}

\noindent
\textbf{Approach.} 
Conventional clustering and segmentation techniques, if taken alone, present intrinsic limitations that make them unsuitable for the cluster-based segmentation problem.   
For example, time-aware clustering methods  
such as ST-DBSCAN \citep{Birant2007} and DensStream \citep{Cao2006} for the clustering of spatio-temporal events and stream data respectively,  do not guarantee the temporal ordering of clusters. 
In particular, ST-DBSCAN 
groups points that are close both in time and space, while DensStream assigns data points a weight based on their temporal freshness to determine whether a group of points is actually a cluster. In both cases, however,   
the clusters may be not temporally ordered. 
By contrast, trajectory segmentation techniques, e.g. \citep{Aronov2015,Kang2004,Buchin2013}, while generating temporally ordered segments, fall short in handling noise. 
On the other hand, simple solutions, such as  introducing a tolerance on the number of  noise points inside a segment, as in \citep{Kang2004,Buchin2013}, are problematic and do not provide guarantees. 

In this research, we investigate a robust solution to the problem of cluster-based segmentation of trajectories  with local noise. 
It is worth noting that in case of no local noise -  the individual simply moves from one residence to another residence - the cluster-based segmentation is fairly straightforward. It  is sufficient to aggregate the points of the sequence in clusters using a  DBSCAN-like technique \citep{Dbscan1996}. Next, once a cluster is created, the  first point in the sequence that cannot be added to such a cluster determines the break of the segment. 
The problem with this solution is that whenever the unclustered points do not have an univocal interpretation, i.e., can represent both  a temporary absence and the definitive departure from the current cluster, such  points cannot be correctly classified until the actual destination or, put in other form, the object's behavior is not known. A different approach is thus needed. 
%
The research presented in this article significantly extends  %
the earlier proposal presented in \citep{Damiani:2014}.
Such proposal comprises the \emph{SeqScan} 
cluster-based segmentation  model, centered on the notion of \emph{presence/absence} in/from the cluster, and the  segmentation algorithm. 
Two key questions remain, however, open: 
how to provide guarantees about the types of patterns that can be discovered; and how to prove the effectiveness of the solution. In this article we address both questions through an in-depth analysis of the properties of the segmentation model and 
the study of suitable evaluation methods. 
Moreover, to highlight the application potential of the technique, we investigate how the framework can be used to facilitate the discovery of additional mobility patterns, which we call \emph{derived}, such as recursive movement patterns \citep{Dataminingbook2014}. Specifically, we propose a technique for location periodicity detection.
To our knowledge this is the first comprehensive framework offering a robust solution to the 
cluster-based trajectory segmentation problem with noise. In summary, the main contributions are as follows. \\

\noindent
\textbf{Contributions.} 
\begin{itemize} 
	\item We provide a rigorous specification of  the \emph{SeqScan} framework, 
	which consolidates the earlier version \citep{Damiani:2014}.  In particular,  we  
	introduce and analyze the property of \emph{spatial separation} of clusters. 
	While such a property is given for grant in classical clustering, it requires a specific characterization in the mobile context where the movement can be recursive, i.e. the same region can be repeatedly visited at different times. As a result, we identify three different mobility patterns, depending on whether clusters are \emph{strongly separated, weakly separated} or \emph{overlapping}. Moreover, we show that the segmentation algorithm splits the trajectory in pairwise separated or weakly separated clusters. In essence, weakly separated clusters describe the circular movement between two consecutive clusters.
	\item We study analytically the relationship between the number of clusters and the temporal parameter - the \emph{presence} -  in order to facilitate the choice of the presence threshold during the segmentation task.  
	\item We evaluate the effectiveness of the segmentation technique 
	by contrasting the SeqScan segmentation with  ground truth. The ground truth consists of a set of labeled trajectories simulating the movement of animals. The trajectory generator is developed by the ecologists co-authoring this article. The evaluation is then conducted 
	through blind experiments.
	\item Finally, we propose a novel technique for the discovery of periodically visited locations, which leverages \emph{SeqScan} and the novel concept of cluster spatial similarity. We contrast our solution with state-of-the-art methods, using real data. The approach is shown to be  effective, simpler to use, and more informative than state-of-the-art methods even in case of periodicity with noise. 
\end{itemize}
The remainder of the article proceeds as follows: Section 2 overviews related research;  Section 3 presents the clustering-based segmentation model and shows the key properties of the model that are at the basis of the \emph{SeqScan} algorithm presented in Section 4 along with the temporal parameter analysis.  Section 5 presents the novel technique for the discovering of periodically visited locations (the \emph{derived} pattern).  The experimental evaluation of both \emph{SeqScan} and the derived pattern discovery technique is presented in Section 6. We conclude with a discussion in Section 7 and final remarks in Section 8.

\section{Related work}
We focus on two
major streams of related research 
concerning the segmentation of trajectories and the detection of stop-and-move patterns, respectively.\\
\\
\textbf{Trajectory segmentation.}
Closely related to our work is the area of  computational movement analysis,
a relatively recent  stream of research,  rooted in  computational geometry and geographical information science, and primarily focused on the concept of \emph{movement pattern}, e.g.  
\citep{Laube08,Dodge08,Buchin2011,Buchin2013,Buchin2014,Aronov2015}. 
Movement patterns describe stereotypical mobility behaviors of individuals located in a geographical space, 
typically characterized in terms of movement attributes or \emph{characteristics}, such as  speed, velocity, direction change rate \citep{Laube08,Dodge08}. 
The segmentation problem is defined as follows: to split a high-sampling rate trajectory into a minimum number of segments such that the movement characteristic
inside each segment is uniform in some sense  \citep{Buchin2011}. 
\citet{Buchin2011,Buchin2013} address the problem of splitting a trajectory based on segmentation criteria that  are \emph{monotone} (decreasing), namely the conditions on the movement characteristics are satisfied by all of the points of the segment. An example is a range condition over speed. 
This approach, however, suffers from a major limitation, in that the criteria of practical interest are  often non-monotone. For example the property regarding  the density in space of the points in a segment is  non-monotone. 

A more flexible framework is presented in \citet{Buchin2014}, which introduces the notion of \emph{stable} criteria, that is criteria  that do not change their validity 'very often'. Such framework can handle Boolean combinations of increasing and decreasing criteria, where \emph{increasing} means that whenever the condition is satisfied by a certain segment it is satisfied also by the segments that contain that segment.
An example combining increasing and decreasing criteria is staying in an area for a minimum duration. Moreover, segmentation tolerates a certain amount of noise, in that a condition can be satisfied except for a fraction of points. Our model differs from this solution in several aspects: the problem is not formulated as an optimization problem, 
moreover we do not seek to provide a generalized segmentation framework. Rather we focus on a specific non-monotone property, density, which cannot be straightforwardly expressed using the aforementioned criteria.  Moreover, we handle noise through the notion of presence/absence, and not counting the points. This solution not only makes the technique more usable (i.e., the presence parameter may have an application meaning), but also makes it much more robust and flexible in case of trajectories with missing points or varying sampling rate.
A different and challenging direction is explored in \citep{Aronov2015}. 
The problem of computing the minimal set of segments based on non-monotonic criteria is shown to be NP-hard in a continuous setting, i.e.,  where trajectories are interpolated linearly between data points and  segments can start and end between data points. 
Though, two  specific criteria are shown to satisfy
properties that make the segmentation problem tractable. Interestingly, both these criteria are related to noise.
In particular, one criterion requires that the minimum
and maximum values of the given attribute
on each segment differ at most for a given amount, while
allowing a certain percentage of outliers. The second
criterion requires the standard
deviation of the attribute value in the segment not to exceed a  threshold value.
For the same reasons discussed above, our goal is substantially different, moreover our solution can be utilized in both discrete and continuous settings.

Segmentation methods are also investigated in other scientific domains,
especially in animal ecology, for the detection of activities or \emph{behavioral states}, e.g.,
foraging, exploration or resting  \citep{Ecology2016}. Such methods rely
on: movement characteristics analysis (in the sense
discussed above), time-series analysis, e.g., \citep{Gurarie2009} and
state-space models, such as hidden Markov models, e.g., \citep{Michelot2017}.
%
There is, however, a substantial difference between these methods and our proposal. Firstly, 
our solution does not target activity recognition. Rather, the goal is movement summarization in presence of noise. 
Secondly, our technique applies to sequences of  temporally annotated data points defined in a metric space, thus the solution is not confined to the analysis of the physical movement in an Euclidean space. 
In this sense, the solution gains in generality and can be employed in arbitrary spaces. \\ 
\\
\textbf{Stop-and-move  detection.} 
A number of techniques for the detection of stops and places  have become  quite popular in recent times \citep{Damiani2017}.
A pioneering technique is the CB-SMoT algorithm \citep{Palma2008}. 
Similarly to DBSCAN, CB-SMoT relies on the notion of $\epsilon$-neighborhood. The $\epsilon$-neighborhood of a point $p$ is however defined along the piecewise linear representation of the spatial trajectory.  It is thus a sub-trajectory consisting of all the points whose distance from $p$ along the line is at most $\epsilon$. Moreover, the parameter  specifying the number of points located in the $\epsilon$-neighborhood, is replaced by the parameter \emph{MinTime} specifying the minimal duration  of the $\epsilon$-neighborhood. %
The substantial limitation of this approach is that it  resembles DBSCAN without in reality  preserving the actual spirit; in particular, this technique is sensitive to noise, because the first point in the sequence that does not belong to the current cluster determines a breakpoint in the sequence. %
Indeed, noise sensitivity is a common issue to all of these techniques, despite the attempts to overcome the problem.  
For example, the  algorithm presented in  \citep{Kang2004}
tolerates up to a maximum number of noise points between two consecutive points in the same cluster. As the noise points exceed this bound, a breakpoint  is added.    
Clearly if the number of noise points is variable or the sampling rate is not regular, this expedient falls short.
Another technique is proposed by Zheng et al. \citep{Zheng2011} as part of a location recommendation system. This solution  is even more restrictive: the first point that is sufficiently far from the beginning of the segment determines a breakpoint. 
Hence if the duration of such a segment  is too brief with respect to the given threshold value, all the points of the segment become noise. %
A common feature of all of these techniques is that they do not offer theoretical guarantees,  are  defined for narrow domains, and lack systematic validation on field. On a different front, the MoveMine project \citep{Movemine2011} presents a technique for the extraction of the periodic movement of an object moving across \emph{reference spots}, where the reference spot is basically defined as a dense region \citep{Movemine2011}. For the detection of reference spots a popular kernel method \citep{Worton1989} designed for the purpose of finding home ranges of animals is used. The notion of reference spot is, however, static and  does not consider time, thus  reference spots do not have a temporal granularity, which instead is one of the  qualifying features of our model. We will come back to the  MoveMine approach, later on in the paper.

\section{ The cluster-based segmentation model}

This section introduces the cluster-based segmentation model. We review the basic concepts and discuss the key properties that are at the basis of the algorithm presented in the next section. Preliminarily, we briefly review  
the DBSCAN cluster model \citep{Dbscan1996}, which provides the ground for the proposed framework and introduce the basic terminology.  
\subsection{Preliminaries and notation}

Consider a database $P$ of points in a metric space. Let $d_s(.)$ be the distance function, e.g. the Euclidean distance, and $\epsilon \in \mathbb{R}$ (i.e. the distance threshold), and $K \in \mathbb{N}$ (i.e. the minimum number of points in a cluster)  the input parameters. 
The cluster model is built on the following definitions \citep{Dbscan1996}:

\begin{definition}[DBSCAN model]
	\emph{
	\begin{itemize}
		\item [-] The $\epsilon$-\emph{neighborhood} of $p\in P$, denoted $N_\epsilon (p)$, is the subset of points that are within distance $\epsilon$ from  $p$, i.e. $N_\epsilon (p)= \{p_i  \in P,  d_s(p, p_i) \leq \epsilon\}$.
		\item [-] A point $p$ is \emph{core point} if its $\epsilon$-neighborhood contains at least $K$ points, i.e. $|N_\epsilon (p)| \geq K$. A point that is not a core point but belongs to the neighborhood of a core point is a \emph{border point}. 
		\item [-]A  point $p$  is \emph{directly density-reachable} from $q$ if $q$ is a core point and $p \in N_\epsilon (q)$. 
		\item [-] Two points $p$ and $q$ are \emph{density reachable} if there is a chain of points $p_1,..,p_n$, $p_1=p, p_n=q$ such that $p_{i+1}$ is directly reachable from $p_{i}$.
		\item [-] Points $p$ and $q$ are  \emph{density connected} if there exists a core point $o$ such that both $p$ and $q$ are density-reachable by $o$. 
		\item [-]   A \emph{cluster} C wrt. $\epsilon$ and $K$ is a non-empty subset  of points
		satisfying the following conditions:
		\begin{itemize}
			\item 1) $\forall p, q$: if $p \in$ C and $q$ is density-reachable from $p$, then $q \in$ C (Maximality)
			\item 2) $\forall p, q \in C: p$ is density-connected to $q$ (Connectivity)
		\end{itemize}
		\item [-] A point $p$ is a \emph{noise} if it is neither a core point nor a border point. This
		implies that noise does not belong to any cluster.
	\end{itemize} }
\end{definition}

An example illustrating the DBSCAN concepts is shown in Figure \ref{dbscan}. Apart from a few peculiar situations\footnote{The property  is not  satisfied by border points. Yet, that is marginal for our work}, the result of
DBSCAN is independent of the order in which the points of
the database are visited  \citep{Dbscan1996}.
Therefore, the algorithm cannot detect sequences of clusters based on some ordering relation over points. 

\begin{figure}[H] 
	\center
	\includegraphics[width=6cm]{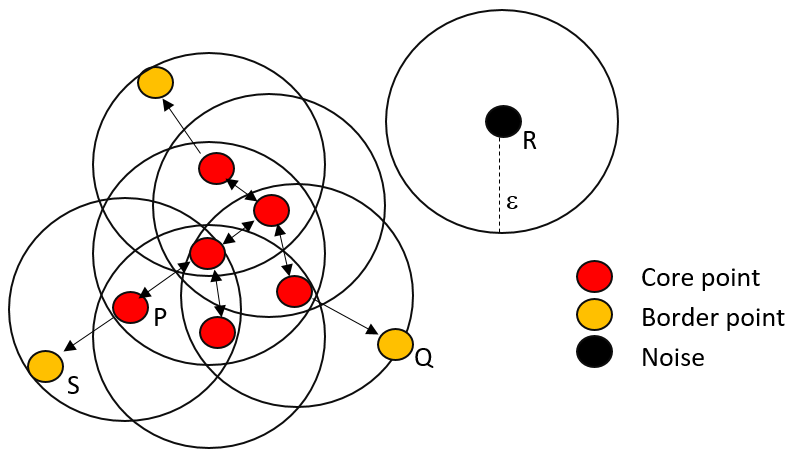}
	\caption{DBSCAN cluster, parameters: $\epsilon,K=4$. $P$ is a core point; $Q$ is a border point because contained in the $\epsilon$-neighborhood of $P$, though not core point;  $Q$ is directly reachable from $P$; $Q$ and $S$ are density connected; $R$ is a noise point.}
	\label{dbscan}
\end{figure} 

Consider now a \emph{trajectory} $T$ of $n$ points $(p_1,t_1),..,(p_n,t_n)$. For the sake of simplicity,  the trajectory is represented by the interval of indices $[1,n]$,  with $i$ indicating the i-esim point $(p_i, t_i)$. 
Besides  the spatial distance $d_s(i,j)$, consider the function $d_t(i,j)$ computing the temporal distance between points $i$ and $j$, respectively. 
A \emph{sub-trajectory} of $T$ is represented by a connected  interval $[i,j] \subseteq T$;  a \emph{segment} by a possibly disconnected interval - a union set of disjoint connected intervals. 
Intuitively, a segment is a sub-trajectory that can have  'holes'.
As shown in Figure \ref{init-notation}, a trajectory is visually represented  as sequence of numbered circles indicating the indexed  points on the plane. The basic notation used throughout the paper is summarized in Table \ref{table:notation}.


\begin{figure}[H]
	\centering
	\subfigure[]
	{\includegraphics[width=5.5cm]{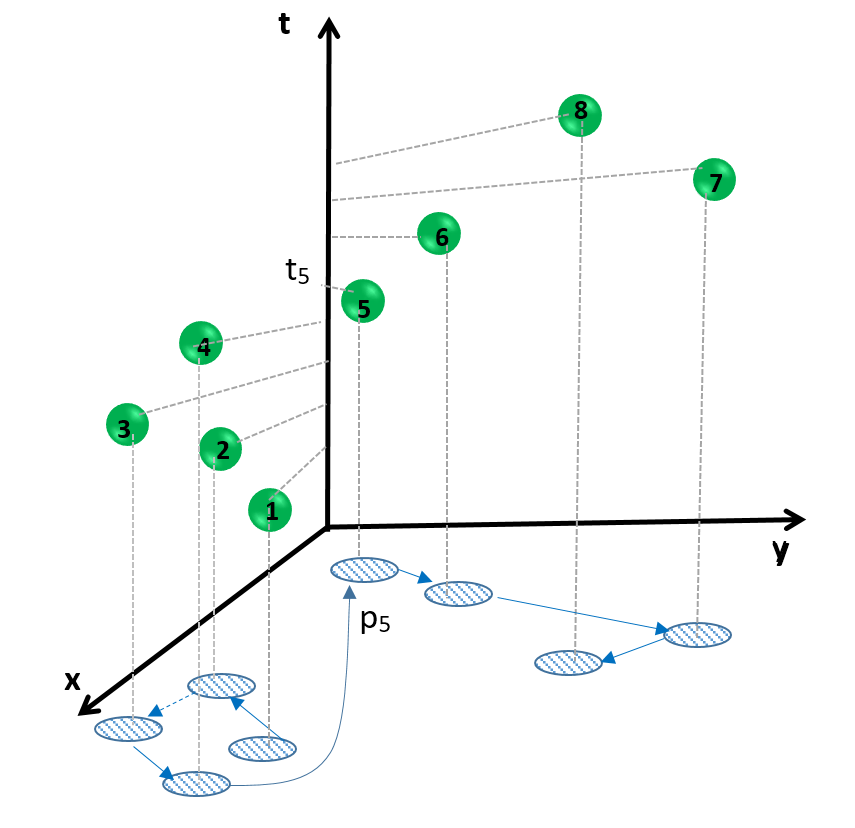}}
	\subfigure[]
	{\includegraphics[width=5.5cm]{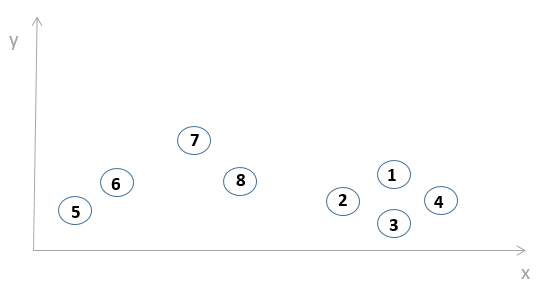}}
	
	\caption{Graphical notation. (a) The trajectory T=[1,8] in the spatio-temporal coordinate system along with the projection on plane and time of the points, i.e. $p_i, t_i$.  The subset $[2,7]$ of T is a sub-trajectory; the 
		subset $[2,3] \cup [5,7]$ a segment. 
		(b) Simplified visual representation used throughout the article.}
	
	\label{init-notation}
\end{figure}

\begin{table}[h]
	\caption{Notation} 
	\centering 
	\small{
		\begin{tabular}{|c | c |} 
			\hline 
			
			\hline 
			I=[i,j] & Sequence of indexed points\\
			$\bigcup_i I_i $ & Segment\\
			$p_j, t_j$ & Spatial point,  temporal annotation \\
			$d_s(i,j)$ & Spatial distance\\
			$d_t(i,j)$ & Temporal distance\\
			K, $\epsilon$ & DBSCAN parameters \\
			$N_{\epsilon}(p)$ & Neighbourhood of point $p$ of radius $\epsilon$\\
			$\delta$ & Presence threshold \\
			$S$ & Cluster/stay region\\
			$\mathcal{P}(S)$ & The value of presence in S \\
			$\mathcal{D}(S)$ & Duration of S\\
			$\mathcal{N}(S)$ & Noise local to  S\\
			$S_i \rightarrow..\rightarrow S_j$ & Sequence of stay regions\\
			$r_j$ & Transition\\
			$|$ & Spatial separation predicate\\
			$\widehat{S}$& Minimal Stay Region in S\\
			
			\hline 
		\end{tabular}
	}
	\label{table:notation} 
\end{table}


\subsection{Cluster and stay region model}
Basically, a stay region is a DBSCAN cluster satisfying a temporal constraint. The key concepts are:

\begin{definition} [Cluster] \emph{
	Given a trajectory $T$,  a cluster $S \subseteq T$ 
	is a segment consisting of  points that, projected on the reference space, constitutes a DBSCAN cluster  (w.r.t. density parameters $\epsilon, K$). Moreover, if the segment is bounded by points $i$ and $j$ then the DBSCAN cluster includes the projection of $i$ and $j$. 
	The set difference  $[i,j] \setminus S$ specifies the corresponding local noise, denoted $\mathcal{N}(S)$. }
\end{definition}

\begin{example}
	\label{ex:stay}
	\emph{
	Consider the trajectory T=[1,7] in Figure \ref{stayregion}.(a) (we omit the coordinate axes for brevity).  If we run DBSCAN on the set of spatial points $p_1,..,p_7$ with parameters K=4 and $\epsilon$ sufficiently small, we obtain that the set $\{p_1,p_2,p_3,p_4, p_7\}$ forms a DBSCAN cluster.  Thus the segment $S=[1,4]\cup[7,7]$ is a cluster in our model, while the points 5 and 6 are local noise. }	
\end{example}

\begin{figure}[H]
	\centering
	\subfigure[]
	{	
		\includegraphics[width=5.5cm]{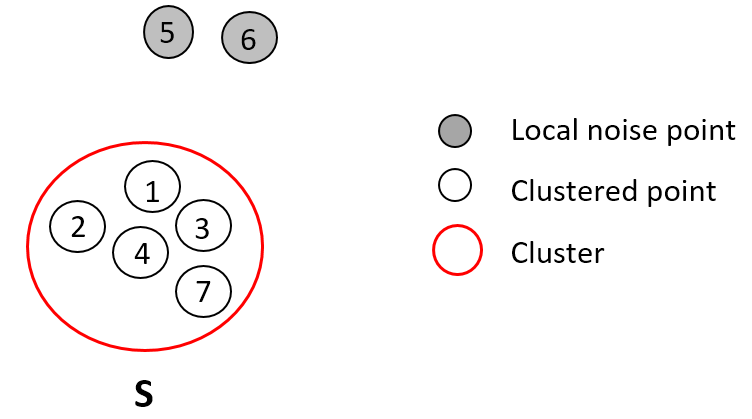}}
	\subfigure[] {
		\includegraphics[width=4.5cm]{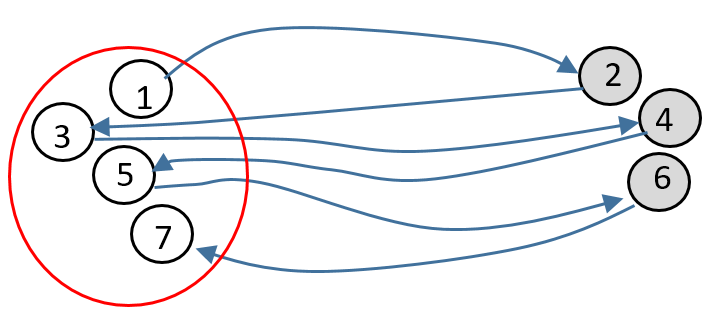}}
	\caption{Graphical notation. (a) The trajectory contains the cluster $S=[1,4]\cup [7,7]$ (open circles) and local noise (grey shaded circles). (b) The object moves back and forth from the cluster. The arrows highlight the flow.}%
	\label{stayregion}
\end{figure}	

A cluster $S$ has a duration $\mathcal{D}(S)$, and a presence $\mathcal{P}(S)$.
The duration $\mathcal{D}(S)$ is simply the temporal distance between the first and last point of the segment, namely $d_t(i,j)$. 
The presence  $\mathcal{P}(S)$ estimates the  residence time in the cluster, with exclusion of the  absence periods, i.e. local noise. Specifically,  the presence of $S$ is defined as the cumulative duration of the connected intervals in $S$. 

\begin{definition}[Presence] \emph{
	Given a cluster $S=S_1\cup..\cup S_m $, with $S_i$ a connected interval,  $\mathcal{P}(S)$ is defined as follows: 
	\begin{equation}
	\label{presence}
	\mathcal{P}(S) =\Sigma_{i=1}^m \mathcal{D}(S_i)
	\end{equation}
	The presence in a cluster $S$ ranges in the interval [0, $\mathcal{D}(S)$].}
\end{definition}

\begin{example}\emph{
	Consider again Figure \ref{stayregion}.(a). Assume for simplicity that the time interval between  consecutive points is 1 time unit. We can see that $\mathcal{P}(S)=3$. The presence in the cluster in Figure \ref{stayregion}.(b) is instead 0 because the object moves back and forth to/from the region  without residing steadily in it.}
\end{example}
This definition of presence relies on the following assumption, that if two consecutive points $i, i+1$ are both members of the cluster then the whole time between $t_i$ and $t_{i+1}$ is assumed to be spent 'inside' the cluster, or more precisely inside the spatial region where the object resides. Conversely, if at least one of the points does not belong to the cluster, then the whole time between $t_i$ and $t_{i+1}$  is spent outside the cluster. 
If the points are
relatively close in time - as we assume -  we postulate  that the presence value can provide a good estimation of the time spent inside the residence.    

\begin{definition}[Stay region]
	\emph{A stay region $S$ is a cluster   satisfying the \emph{minimum presence constraint} defined as:  
	\begin{equation}
	\mathcal{P}(S) \geq \delta
	\end{equation}
	where $\delta \geq 0$ is the \emph{presence threshold}. }
\end{definition}
\begin{example}
	\emph{The cluster $S$ in Figure \ref{stayregion}.(a) is a stay region if $\delta \leq 3$. Conversely, if $\delta > 3$, the time spent in $S$ is not enough for the cluster to represent an object's residence.}
\end{example}
A property, which will be recalled later on, is that  the minimum presence constraint is \emph{monotone}, i.e. if  the constraint  is satisfied by stay region $S_1$, then it is also satisfied by any stay region $S_2$ such that $S_1\subset S_2$.


We now turn to consider sequences of stay regions and 
the notion of segmentation.
A segmentation is a partitioning of a trajectory in a sequence of disjoint segments that can represent either stay regions or segments of unclusters points. Segmentation is defined with respect to the three parameters: $K,\epsilon, \delta$. More formally:

\begin{definition}[Cluster-based segmentation] \emph{Let $T=[1,n]$ be a trajectory and $K, \epsilon, \delta$ the segmentation parameters.
	A segmentation  is a set of disjoint segments $\{S_1,..,S_m\} \cup \{r_0,.., r_m\}$, covering the whole trajectory where: }
	\begin{itemize} 
		\item 	\emph{$S_1,..,S_m $  are stay regions satisfying the following conditions:}
		\begin{itemize} 
			\item \emph{Stay regions are temporally separated}
			\item \emph{ Stay regions are of maximal length, i.e. any point that can be included in the cluster without compromising the temporal separation of the clusters is included. }
		
		\end{itemize} 
		\item \emph{$r_0,..,r_m$ are possibly empty \emph{transitions}. Transitions do not include any point that can be added to stay regions.}
	\end{itemize}
	\emph{A segmentation can be represented as follows: 
	$\xrightarrow{r_0} S_1 \xrightarrow{r_1} S_2 ..\xrightarrow{r_{m-1}} S_m \xrightarrow {r_{m}} $.  
	The sequence of stay regions is referred to as \emph{path}.}
\end{definition}

\begin{example} \emph{Figure \ref{path} shows a segmentation comprising two stay regions w.r.t. $K=4, \epsilon, \delta=0$:  $S1 \xrightarrow{r_1} S2\xrightarrow{r_2}$ 
	where: $S_1=[1,4], S_2=[5,7]\cup [11,11]$, $ r_1=\emptyset, r_2=[12,13]$. 
	The stay regions are maximal.
	The object moves straightforwardly  from $S_1$ to $S_2$, next it experiences a period of absence from $S_2$ and finally leaves it.}
	
\end{example}

\begin{figure}
	\center
	\includegraphics[width=11cm]{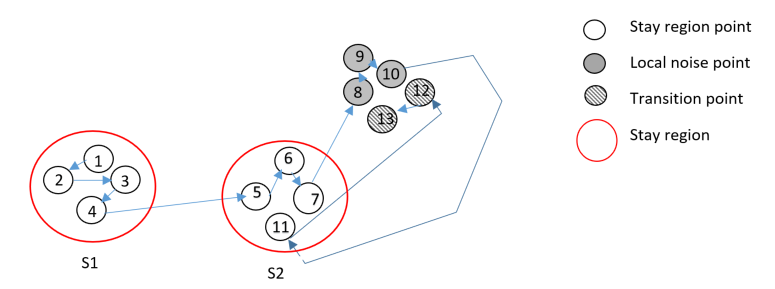}
	\caption{Trajectory segmentation: the segmentation consists of two stay regions $S_1, S_2$ w.r.t. $K=4, \epsilon, \delta=0$. For the sake of readability, the stay regions are labeled $S1,S2$. }
	\label{path}
\end{figure}

\subsection{Properties of the model: spatial separation}
Following the definition of segmentation, the stay regions in a path are temporally separated. 
A straightforward question  is whether the stay regions are also separated in space.  Indeed spatial separation is a natural property of 'conventional' clusters. In our model, however, the situation is  slightly different because the trajectory describes an evolving phenomenon, therefore the notion of spatial separation requires a more precise characterization that takes time into account.
We begin with a  general definition of spatial separation between two stay regions, next we discuss some key properties that are at the basis of the algorithm presented next. Basically, a stay region $S_2$ is spatially separated by $S_1$ if no  point exists in either $S_2$ or in the corresponding local noise that is reachable - in the DBSCAN sense - from a point of $S_1$. 
In other words, while residing in $S_2$ the moving object has to be sufficiently far from $S_1$ even during the periods of absence. Formally:

\begin{figure}[H]
	\center
	\subfigure[]{
		\fbox{\includegraphics[width=4.5cm]{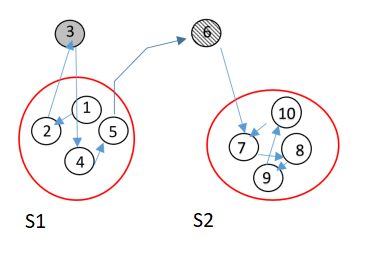}}}
	\subfigure[]{
		\fbox{\includegraphics[width=4.9cm]{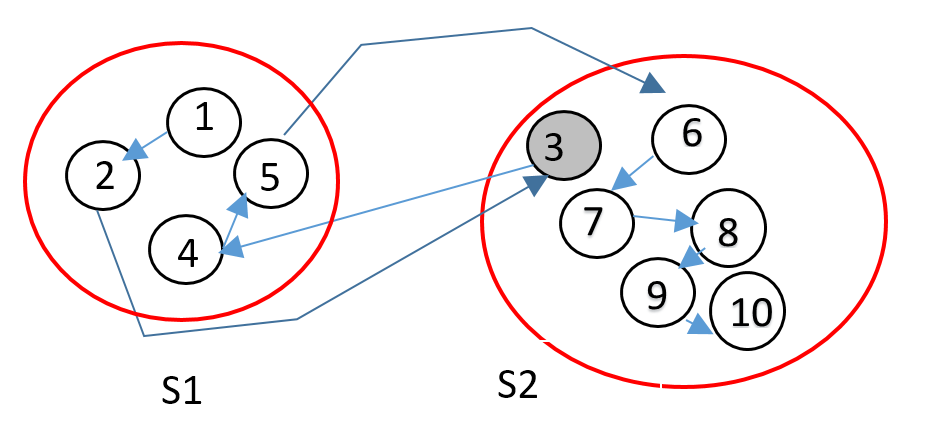}}}
	\caption{Spatial separation of stay regions. (a) Two spatially separated stay regions $S_1, S_2$ (labeled $S1,S2$ for clarity) with $S_1=[1,2]\cup [4,5]$ and $S_2=[7,10]$. Point $3$ is a noise point local to $S_1$, point $6$  a transition point; (b) Asymmetry of the spatial separation relationship: $S_2$ is separated by $S_1$ but the vice-versa is not true.}
	\label{separate}
\end{figure}

\begin{definition}[Spatial separation] \emph{
	\label{defsep}
	Let $S_1$, $S_2$ be two stay regions in a path, non-necessarily consecutive. 
	We say that $S_2$ is spatially separated by $S_1$, denoted $S_2 \vert S_1$,  
	if no point $p \in S_2 \cup \mathcal{N}(S_2)$ 
	belongs to the  $\epsilon$-neighborhood of  any core point $q \in S_1$ (i.e. $p$ is not reachable from $S_1$).  
	Two stay regions that are not spatial separated are said to \emph{overlap}.}
\end{definition}
It can be shown that the relationship of spatial separation between stay regions is asymmetric. 
\begin{equation}
S_j | S_i \nRightarrow  S_i | S_j
\end{equation}
\begin{example} \emph{
	Figure \ref{separate}.(a) shows the segmentation of the trajectory $T=[1,10]$ in two stay regions $S_1,S_2$ connected through the transition  $\{6\}$. $S_2$ is separated from $S_1$ and viceversa. Figure  \ref{separate}.(b) shows an example of stay regions that are not separated. }
	
\end{example}

The next concept is that of Minimal Stay Region (MSR). This concept is at the basis of the cluster-based segmentation algorithm. 
In essence, the MSR is the 'seed' of a stay region. Formally: 


\begin{definition}[Minimal Stay Region] \emph{The MSR of a stay region $S$ (w.r.t. $\epsilon, K. \delta$), denoted $\widehat{S}$,  is the stay region of minimal length contained in $S$ that is created first in time.} 
\end{definition}

\begin{example}
	\emph{The trajectory [1,7] in Figure \ref{msr}.(a) is a stay region (w.r.t. K=4 $\epsilon, \delta=0$) and the MSR is the connected interval [2,5]. Note that the sub-trajectory [1,4] is not a stay region because it does not contain a cluster of at  least 4 elements. The cluster is only created at time $t_5$. At that time, the cluster of minimal length is [2,5].}
\end{example}
\begin{figure}[H] 
	\center
	\subfigure[]{
		\fbox{\includegraphics[width=3.3cm]{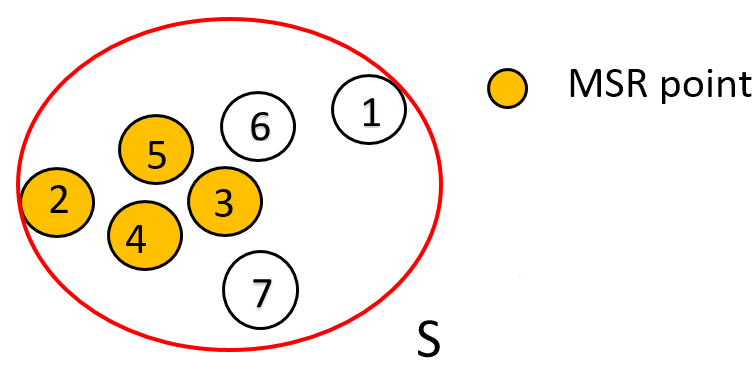}}	}	\subfigure[]{ 	
		\fbox{\includegraphics[width=3.7cm]{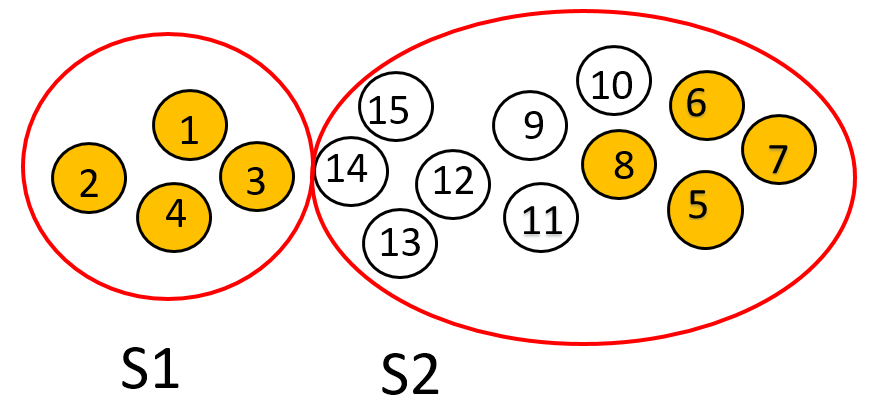}}}
	\subfigure[]{
		\fbox{\includegraphics[width=3.5cm]{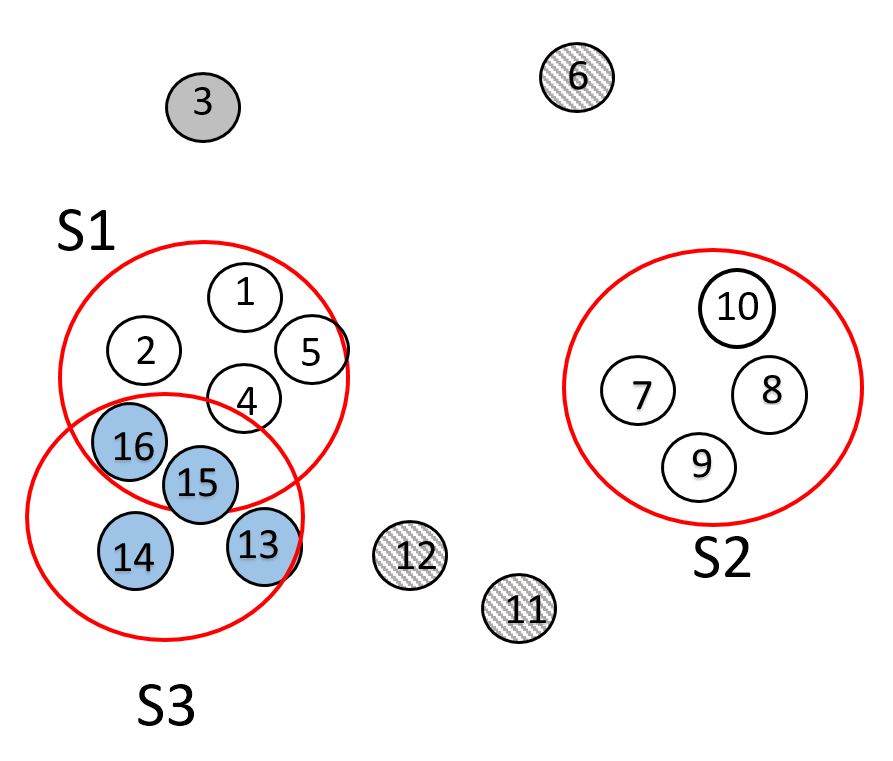}}}	
	\caption{Minimal Stay Region (MSR). (a) The stay region S=[1,7]  contains the MSR [2,5] (yellow points); (b) Weakly separated stay regions: whilst the MSRs are separated, the points of $S_2$ are reachable from $S_1$; (c) Overlapping non-consecutive stay regions: $S_1$, $S_3$ (to ease readability $S_3$ is colored)}
	\label{msr}  
\end{figure}

From the definition of segmentation, 
we can derive the following Theorem stating a necessary condition on the spatial separation of consecutive stay regions:
\begin{theorem} 
	\label{mainth}
	For any pair of consecutive stay regions	$S_i \xrightarrow{r} S_{i+1}$, it holds that the MSR $\widehat{S}_{i+1}$ is spatially separated from $S_i$, namely $\widehat{S}_{i+1}|S_i$. 
\end{theorem}

\begin{proof} \emph{
	The proof is by contradiction. Suppose that $\widehat{S}_{i+1}$ is not separated from $S_i$. Based on Definition \ref{defsep},  at least one point  exists in the set  $\widehat{S}_{i+1}  \cup \mathcal{N}(\widehat{S}_{i+1})$ that is directly reachable from $S_i$. Let $j$ be the point with lowest index reachable from $S_i$ and consider the segment $S_i \cup \{j\}$.
	We see that:  a) no other cluster can exist in between $S_i$ and $j$ because $j$ is the lowest index; 
	b) the segment satisfies the minimum presence constraint.  Thus  $S_i \cup \{j\}$  is  a stay region in the path. 
	However, that contradicts the assumption that $S_i$ is a cluster of maximal length. Therefore $\widehat{S}_{i+1}$ must be separated from the previous stay region, which is what we wanted to demonstrate $\diamond$
	}
\end{proof}

The next two corollaries provide a motivation for specific mobility behaviors that can be observed in a trajectory. In particular  Corollary  \ref{cor1} states that two consecutive stay regions can spatially overlap for some time.  The intuition is that  when the  object leaves a residence for another residence, after a while it can start moving back gradually to the previous region. 
Corollary \ref{cor2} states that 
two non-consecutive stay regions, even identical in space,  but frequented in different periods, are treated as two different stay regions. In other terms,  non-consecutive stay regions can  overlap.  Formally:

\begin{corollary}
	\label{cor1}
	Let  $S_i \xrightarrow{r} S_{i+1}$ be two consecutive stay regions. The points in $S_{i+1}$ following in time the minimal stay region $\widehat{S}_{i+1}$ 
	may be not spatially separated from $S_i$. We refer to this property as weak spatial separation.
\end{corollary}

\begin{corollary}
	\label{cor2}
	Two non-consecutive stay regions can overlap
\end{corollary}
Weakly-separated and overlapping regions are exemplified in Figure \ref{msr}.(b) and  \ref{msr}.(c), respectively. 

Finally, the following theorem  reformulates the notion of path in more specific terms. It follows straightforwardly  from the above  results.
\begin{theorem}
	\label{thcore}
	A path in a trajectory (w.r.t. $\epsilon, K, \delta$) is a sequence of temporally separated stay regions of maximal length %
	and possibly  pairwise weakly spatially separated
\end{theorem}

An important property is the following. 

\begin{corollary}
	The path in a trajectory may be not unique. 
\end{corollary}

\begin{figure}
	\center
	\subfigure[]{
		\fbox{\includegraphics[width=5cm]{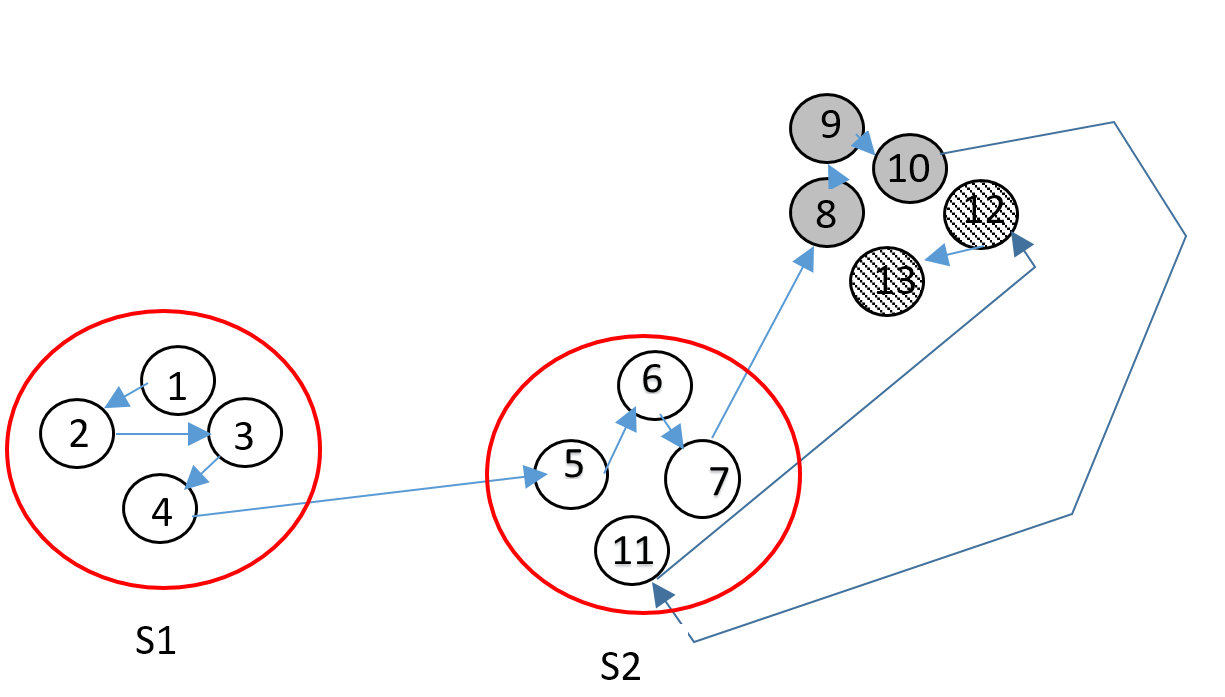}}}
	\subfigure[]{
		\fbox{\includegraphics[width=5.5cm]{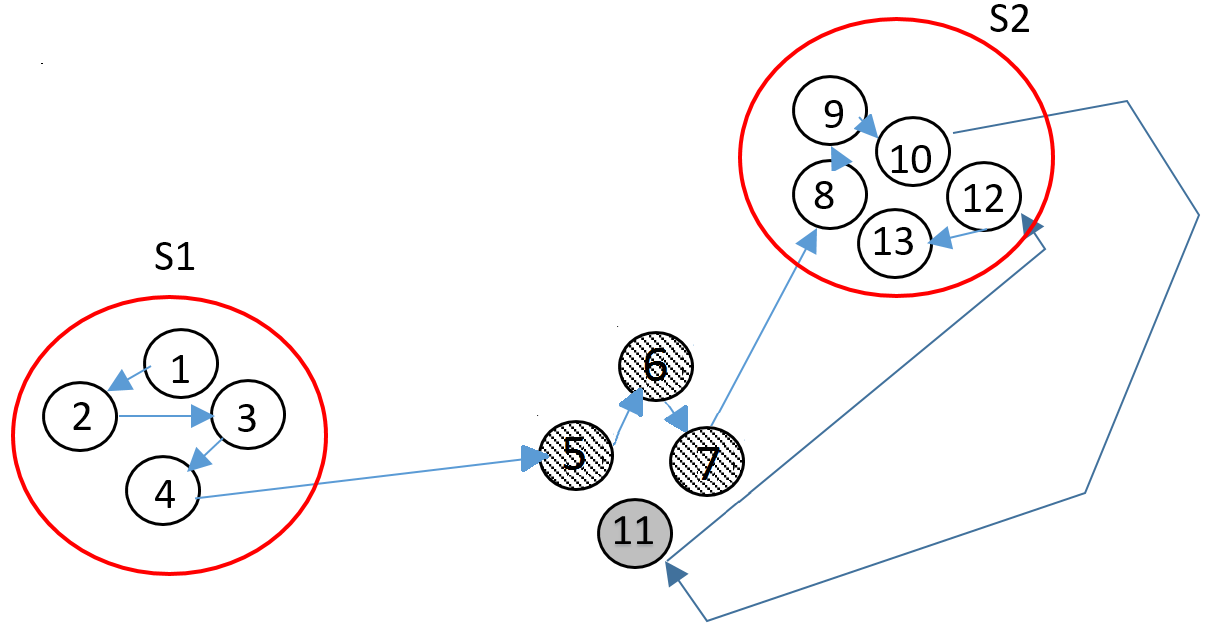}}}
	\caption{Two different segmentations for the same  trajectory (w.r.t. $K=4, \delta=0$, $\epsilon$). (a) $S_2$ is the stay region detected  first; (b) $S_2$ is the region containing the highest number of points.} 
	\label{fig:problem}
\end{figure}
\begin{example}
	\emph{Figure \ref{fig:problem} shows two different segmentations for the same trajectory, both including two stay regions $S_1, S_2$.  Such regions are however selected based on different criteria.  (a) $S_2$ is the stay region that is created first in time after $S_1$. (b) $S_2$ is the stay region with the highest number of points after $S_1$. }
\end{example}


\section{The SeqScan algorithm and the analysis of the temporal parameter}
Based on the  model, we  describe the SeqScan algorithm for the cluster-based segmentation of a trajectory. Next we discuss the problem of how to choose the temporal parameter.

\subsection{The segmentation algorithm}
We have seen that the segmentation may be not unique. Therefore, we need to specify the criterion based on which selecting the stay regions of the sequence. We choose the following criterion: 
following the temporal order, the first cluster that satisfies the minimum presence constraint becomes the next stay region in the sequence.  
This criterion has an intuitive explanation: 
\emph{an object resides in a region  until another attractive residence is found}.  The resulting path is called hereinafter \emph{first path}.  The \emph{SeqScan} algorithm for the extraction of the first path is presented in the following:

\paragraph{Overview.}

\begin{algorithm}[t]
	\caption{SeqScan}

	\begin{algorithmic}
		
		\Procedure {SeqScan}
		{
			In: $T=[1,n], \epsilon,K,\delta$; Out: stayRegionsSet}

		\State $\textit{stayRegionsSet} \gets \emptyset$
		\State $\textit{C} \gets \emptyset$ \Comment{the context of the active stay region} 
		\State $\textit{P} \gets \emptyset$ \Comment{pool of points for MSR search}
		\State $\textit{i} \gets \textit{1}$  \Comment{current scan index}
		\State $\textit{activeStayRegion} \gets \emptyset$ 
		\While{$i\leq n$}
		\State $\textit{Insert(C,i)}$ \Comment   incremental clustering of C
		\If {$\textit{\textbf{canExpand}(activeStayRegion, i, C)}$} \Comment{if $i$ can be added}
		\State $P \gets \emptyset$ \Comment{pool reset}
		
		\Else
		\State $\textit{Insert(P,i)}$ \Comment  incremental clustering of P
		\State $\textit{nextStayRegion} \gets \textit{\textbf{findMSR}$(P)$}$ \Comment{find the next minimal stay region}
		\If {$\textit{nextStayRegion}\neq \emptyset$ }
		\State \textit{stayRegionsSet} $\gets$   \textbf{\textit{addStayRegion}} \textit{(activeStayRegion)}	
		\State $\textit{activeStayRegion} \gets \textit{nextStayRegion}$
		\State $C \gets P$		\Comment{set the context for the new MSR}
		\State $P \gets \emptyset$ \Comment{reset of the  pool}
		\EndIf
		\EndIf
		
		\State $\textit{i} \gets \textit{i+1}$
		\EndWhile

		\EndProcedure
	\end{algorithmic}
	\label{algoSeq}
\end{algorithm}

The SeqScan algorithm  extracts from the trajectory $T$  a sequence of stay regions, based on the three input parameters $K, \epsilon, \delta$. The noise points can then be obtained by difference from $T$ and straightforwardly classified in transition and local noise points. 
The algorithm scans the trajectory, iterating through the following phases:  i) Find a Minimal Stay Region. Such  MRS becomes the \emph{active} stay region. ii) Expand the active stay region. 
iii) Close  the active stay region. Once closed, a stay region cannot be  expanded anymore.
More specifically:
\begin{itemize}
	\item [(i)] Search: the algorithm runs the DBSCAN algorithm on the spatial projection of the input sequence (i.e. spatial points). 
	The clustering algorithm  processes  the  points in the temporal order, progressively aggregating points in  clusters. The  cluster that for first in time satisfies the minimum presence constraint  determines the new MSR $S_i$. $S_i$ becomes the \emph{active} stay region.
	
	\item [(ii)] Expand: the active stay region  is expanded.  
	The question at this stage is how to determine the end of the expansion and thus the break of the segment. We recall that, in the stay region model, a point that is not reachable from a cluster can indicate either a temporary absence, or a transition or  be an element of a more recent stay region. Therefore such a point cannot be correctly classified,  until the movement evolution is known. %
	The proposed solution is detailed next.
	\item [(iii)] Close: the active stay region is deactivated, or closed, when a more recent MSR is found. Such MSR becomes the  new active stay region $S_{i+1}$. A closed stay region is simply a stay region in its final shape.
\end{itemize} 

\paragraph{Detailed algorithm.}

The pseudo-code is reported in Algorithm \ref{algoSeq}. 
At each step, the algorithm  tries first to expand the active stay region $S_j$, and if that is not possible, tries to create a new MSR $S_{j+1}$. 
To perform such operations, the algorithm  maintains two  different sets of points that are clustered incrementally using the Incremental DBSCAN algorithm \citep{Ester1998}. These sets are called $C$ and $P$, respectively.  $C$ represents the \emph{Context} of the active stay region $S_j$, namely the set of  points that at each step can be used for the expansion of the cluster. Such points follow the previous stay region in the sequence, thus $C$ is separated from $S_{j-1}$. The set $P$ is the \emph{Pool} of points following in time the active stay region and representing the space where to search for the next MSR. Accordingly $P$ is temporally separated from $S_j$.
When a new point is added to either C or P, the set is clusterized incrementally using the Incremental DBSCAN technique\citep{Ester1998}. The  processing of the input point $i$ is thus as follows: 
\begin{itemize}
	\item  $i$ is  first added to the Context $C$. If the point can be added to the current cluster, then the stay region is prolonged to include the point. Next the Pool is reset to the empty set.
	\item if $i$ cannot be added to the active stay region, then $i$ is added to the Pool $P$. If a MSR can be created out of P then such a MSR becomes the new active stay region $S_{i+1}$. Accordingly, $S_i$ is closed,  the Pool $P$ becomes the Context for $S_{i+1}$ and $P$ is reset to the emptyset. 
\end{itemize}
The run-time complexity of SeqScan is that of Incremental DBSCAN, i.e. $O(n^2)$ \citep{Ester1998,Tao2015}. 
%
\begin{example}
	\emph{
	We illustrate the algorithm through an example focusing on the expansion part. Consider a trajectory T=[1,13]. The clustering parameters are set to: K=4,  $\delta=0$, $\epsilon$ sufficiently small. We analyze the expansion of the  active stay region $S_1$, starting from the MSR depicted in the Figure \ref{exalgo}.(a). For every subsequent point, we report the change of state defined by the triple: active stay region, C, P. }
	
	\begin{enumerate}
		
		\item 
		\emph{State: $S_1=[1,1] \cup [3,5]$,  C= [1,5], P= $\emptyset$}
		
		\item \emph{Read point: 6.  The point cannot be added to the active stay region, thus the state is: $S_1$, C= [1,6], P=[6,6]}
		\item \emph{Read  point: 7. The point cannot be added to the active stay region. State: $S_1$,
		C= [1,7], P=[6,7]}
		\item \emph{Read point 8. The point can be added to the active stay region. State:  $S_1=[1,1] \cup [3,5]\cup[8,8]$, C=[1,8], P=$\emptyset$.}
		\item \emph{Read point 9. The point cannot be added to the active stay region. State: $S_1$, C=[1,9], P=[9,9]}
		\item  \emph{Read point 10, as above. State: $S_1$, C=[1,10], P=[9,10]}
		\item  \emph{Read point 11, as above. State: $S_1$, C=[1,11], P=[9,11]}
		\item  \emph{Read point 12, as above. State: $S_1$, C=[1,12], P=[9,12].}
		\item  \emph{Read point 13. The point cannot be added to the active stay region. However, the  insertion of the point in P, i.e. P=[9,13], generates a new stay region. Accordingly  $S_1$ is closed, $S_2=[10,13]$, C=[9,13], P=$\emptyset$. The scan is terminated}
	
	\end{enumerate}

	\emph{The final stay regions are thus $S_1=[1,1] \cup [3,5] \cup [8,8]$ and $S_2=[10,13]$. The noise  points can then be  classified. Points  2, 6,7  fall in the temporal extent [1,8] of $S_1$ thus are local noise; point 9 is a transition point}
	\begin{figure}[H] 
		\center
		\subfigure[State at step 5. Next input point: 6]
		{\fbox{\includegraphics[height=2.6cm]{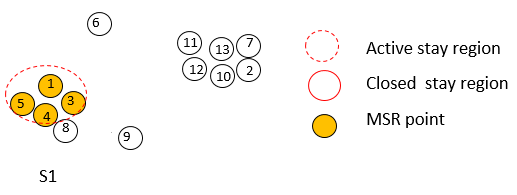}}}
		\subfigure[State at step 13. Next the trajectory terminates and $S_2$ is closed]
		{\fbox{\includegraphics[height=2.6cm]{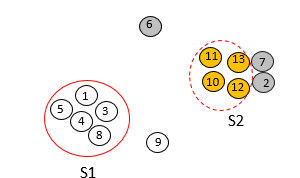}}}
		\caption{SeqScan processing. (a) $S_1$ becomes the active stay region; (b) the creation of $S_2$}
		\label{exalgo}
	\end{figure}

\end{example}
\begin{theorem}
	Algorithm \ref{algoSeq} computes the first path, if any path exists in the input trajectory.
\end{theorem}
\begin{proof} \emph{
	We need to prove that the resulting stay regions  are maximal, temporally separated and pairwise weakly  spatially separated. Moreover, every stay region is the first to be created after the previous one. The reasoning is as follows.
	The algorithm builds a stay region by expanding the first MSR that is encountered.  Further the stay region is expanded based on the Context $C$ that includes all of the points following the end of the previous stay region, thus all those that potentially can be added to the stay region. 
	The stay region is thus maximal and the first to be created.  The stay regions are  weakly spatially separated  because when a MSR  is found in the Pool $P$, it cannot contain points 'close' to the previous stay region (otherwise such point would be added to that stay region). Yet, once the MSR is created, the subsequent points that are added to the active stay region during the expansion phase can be located even in close proximity with the preceding stay region.  The stay regions are  also temporally separated because stay regions are  created and next expanded from the two sets $P$ and $C$ that by definition are separated from the previous stay region.}
\end{proof}


\subsection{Choice of the 'presence' parameter}
\label{sec:param}

SeqScan requires in input the presence threshold parameter $\delta$. A question of practical relevance is how to set the value for this parameter.
In this section
we analyze the relationship between  $\delta$ and the number of stay regions. We recall that the presence threshold somehow constrains the temporal granularity of the stay regions in the path. The purpose of this analysis is to help determine the desired level of temporal granularity. The Optics system \citep{Optics1999} has a similar intent, though applied to a density parameter of DBSCAN, and not to time as in our case.

Consider  a trajectory $T=[1,n]$ and let  $f_T: [0,\mathcal{D}(T)] \rightarrow \mathbb{N}$ be  the function yielding the number of clusters in  $T$, for values of $\delta \in [0,\mathcal{D}(T)]$. The other parameters $K, \epsilon$ are fixed.
We show  that the number of stay regions remains constant for values of $\delta$ ranging in properly defined intervals. That is, the function $f_T$ has a step-wise behavior. The proof is in two steps: Lemma \ref{lemma1} demonstrates that the number of stay regions that we obtain is the same if we consider the sequences of MSRs in place of stay regions. This result is next used in Theorem \ref{thmono} to prove that  the function $f_T$ is step-wise.

\begin{lemma}
	\label{lemma1}	
	Let $\mathcal{\widehat{S}}=\widehat{S_1},\widehat{S_2},..$ and $\mathcal{\widehat{S}'}=\widehat{S_1}',\widehat{S_2}',..$ the two sequences of minimal stay regions in  $T$ obtained  with $\delta=p$ and $\delta=p'$ respectively (w.r.t.  density parameters $K, \epsilon$). 
	These two sequences $\widehat{S}$ and $\widehat{S}'$ are identical iff the corresponding stay regions are identical.
	$$\mathcal{\widehat{S}'}= \mathcal{\widehat{S}} \Leftrightarrow  \mathcal{S}' = \mathcal{S}$$
\end{lemma}
\begin{proof} \emph{
	We prove the implication $\mathcal{\widehat{S}'}= \mathcal{\widehat{S}} \Rightarrow  \mathcal{S'} = \mathcal{S}$ (the other way is trivial).  Consider for a generic index $i$, the equality $\widehat{S_i}'=\widehat{S_i}$. By definition of minimal stay region, every element of $ S_i$ is reachable - in the DBSCAN sense - from $ \widehat{S_i}$. Thus every element of $S_i$ is also reachable by $\widehat{S_i}'$. Similarly  every element of  $ S_i'$ is reachable from $\widehat{S_i}$. Thus the two stay regions  $S_i'$ and $S_i$ are identical. $\diamond$} 
\end{proof}

The next Theorem shows that the number of stay regions remains constant for values of $\delta$ ranging in well-defined intervals. This provides the ground for the  computation of $f_T$.
\begin{theorem}
	\label{thmono}	
	Consider a trajectory $T$. Let $S_1,..,S_m$  be the sequence of stay regions obtained with parameter $\delta=p$ and  $p_1,..,p_m \geq p$ the presence value in the MSRs  $\widehat{S}_1,..,\widehat{S}_m$. It holds that:
	\begin{equation}
	\forall p' \in [p, \min_{i\in [1,m]} {p_i}], f_T(p')=f_T(p)
	\end{equation}
\end{theorem}
\begin{proof}
	\emph{
	Let $p' \in (p, \min_{i\in [1,m]} {p_i}]$ (otherwise the case is trivial). 
	We want to prove that $f(p')=f(p)$.
	We recall that the SeqScan algorithm scans the trajectory
	until a minimal stay region is found, hence such a cluster is expanded until a new,
	spatially separated MSR is detected. As the value $p' > p$ is lower
	than the presence in any minimal stay region, none of the  stay regions in $f_T(p)$ is filtered out. Similarly, no additional minimal stay region can be found. The sequence of minimal stay regions is identical, and thus  for Lemma \ref{lemma1} also the sequences of stay regions, i.e. $f_T(p)=f_T(p')$ that is what we wanted to demonstrate.}
\end{proof}

\begin{algorithm}
	\caption{Computing the function $f_T$} 
	\label{algo:f}	
	\begin{algorithmic}		
		\Procedure {$f_T$}
		{In: $T=[1,n], \epsilon,K$, step; Out: SetOfPairs}
		\State \textit {min $\gets$ 0, SetOfPairs $ \gets \emptyset$}
		\State \textit{$\delta \gets 0$}
		\While {$min \neq$ -1} 
		
		\State \textit{$SeqScan(T, \epsilon, K,\delta, S=[S_1,..,S_m])$}
		\If {\textit{S $\neq \emptyset$ }}
		\State \textit{$min \gets$ minumum presence of minimal stay regions from $[\widehat{S}_1,..,\widehat{S}_m]$}
		\State \textit{Add ([$\delta, min$], m) to  SetOfPairs}
		\State \textit{$\delta \gets min+ \theta$}
		\Else 
		{\textit{ min=-1}}
		\EndIf
		\EndWhile
		\EndProcedure
	\end{algorithmic}
\end{algorithm}
The function $f_T$ is given a constructive definition in Algorithm \ref{algo:f}. This algorithm  runs SeqScan multiple times with different values of the parameter $\delta$ until the number of resulting stay regions is 0.
The presence threshold is initially set to $\delta=0$. We illustrate the iterative process as follows. After the first run, SeqScan returns a sequence $S$ of stay regions, based on which, the minimum value of presence in the respective MSRs is computed. Such a value, say $v_1$, forms the upper bound of the first interval $I_1=[0,v_1]$. For values of $\delta$ falling in such interval,  the number of stay regions is $|S|$. Next, SeqScan is run with  $\delta = v_1 + \theta$ ( with $\theta >0$ is a small constant used to handle the discontinuity) to possibly determine the second interval $I_2$. The  process iterates until the terminating condition is met.


\subsection{Examples: SeqScan at work}
We conclude this section illustrating the three major typologies of patterns that can be detected by SeqScan: the linear ordering of clusters with local noise and transitions; weakly separated consecutive clusters; and overlapping non-consecutive clusters. 
The patterns are shown in Figure  \ref{case1}. The  trajectories have been generated manually.	
\begin{itemize}
	\item \textbf{Pattern 1}: linear ordering of stay regions with local noise.
	%
	The trajectory is displayed in Figure \ref{case1}.(a). 
	SeqScan is run with  parameters $\delta=0$, $\epsilon=70, K=20$. The result is shown in Figure \ref{case1}.(b). It can be seen that the segmentation correctly identifies the two clusters, the transition and some local noise associated with one of the clusters.
	
	%
	\item \textbf{Pattern 2}: weak separation of consecutive  stay regions. 
	The trajectory  in Figure \ref{case1}.(c) contains  two  clusters that are  spatially separated only for a limited period of time. In particular it can be noticed that the object moves back from the second region to the initial region. We run with the same parameters as above,  we obtain the two stay regions reported in Figure \ref{case1}.(d). The two stay regions are evidently not disjointed and this is coherent with the fact that consecutive stay regions can be weakly separated, in accordance with Corollary \ref{cor1}.

	\item \textbf{Pattern 3}: non-spatially separated  stay regions.
	The trajectory in Figure \ref{case1}.(d) exemplifies the case of an object moving back and forth between two regions. The Figure shows are 4 clusters. Of these, the non-consecutive clusters are overlapping. Coherently with Corollary \ref{cor2}, SeqScan  detects the correct sequence of clusters. 

\end{itemize}
\begin{figure} [H]
	\centering
	\subfigure[Pattern 1]
	{\fbox{\includegraphics[height=2.5cm]{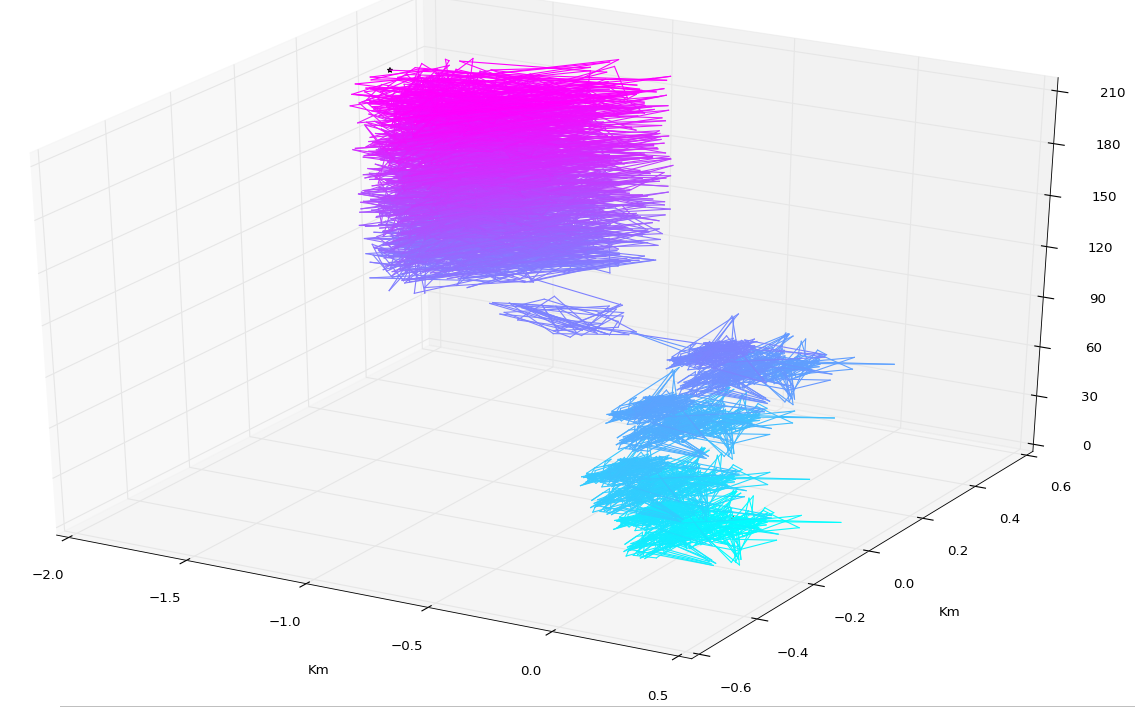}}}
	\subfigure[]
	{\fbox{\includegraphics[height=2.5cm]{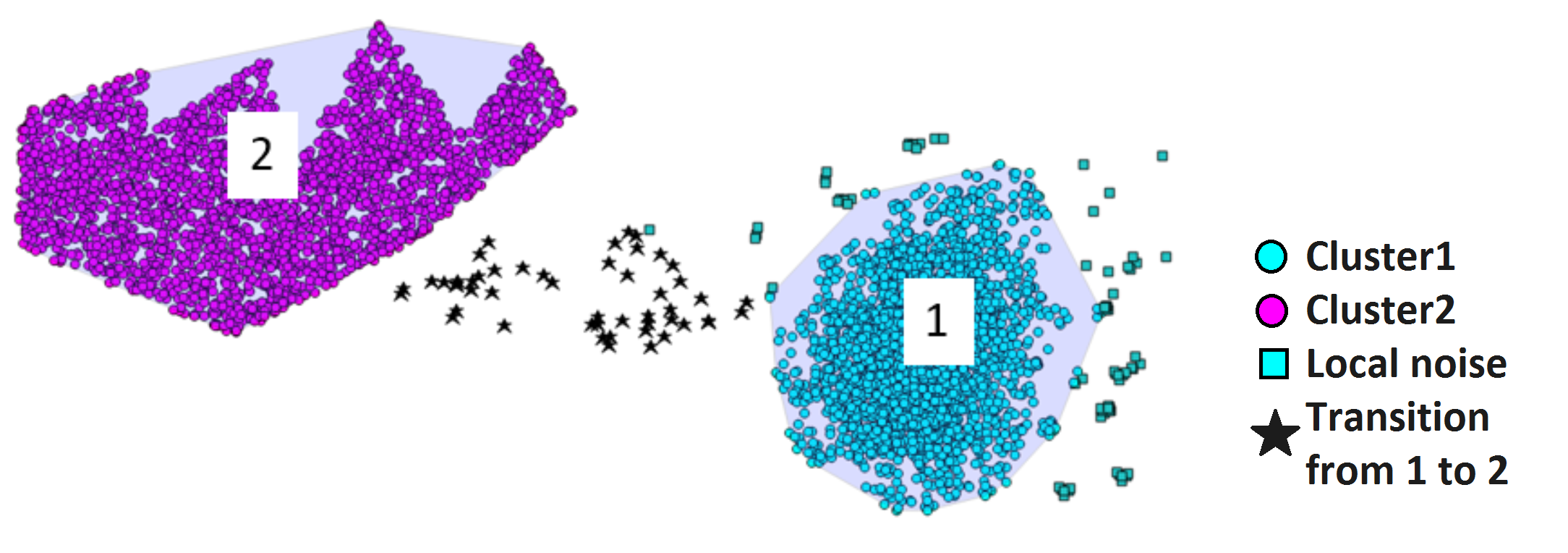}}}
	\subfigure[Pattern 2]	
	{\fbox{\includegraphics[height=3cm]{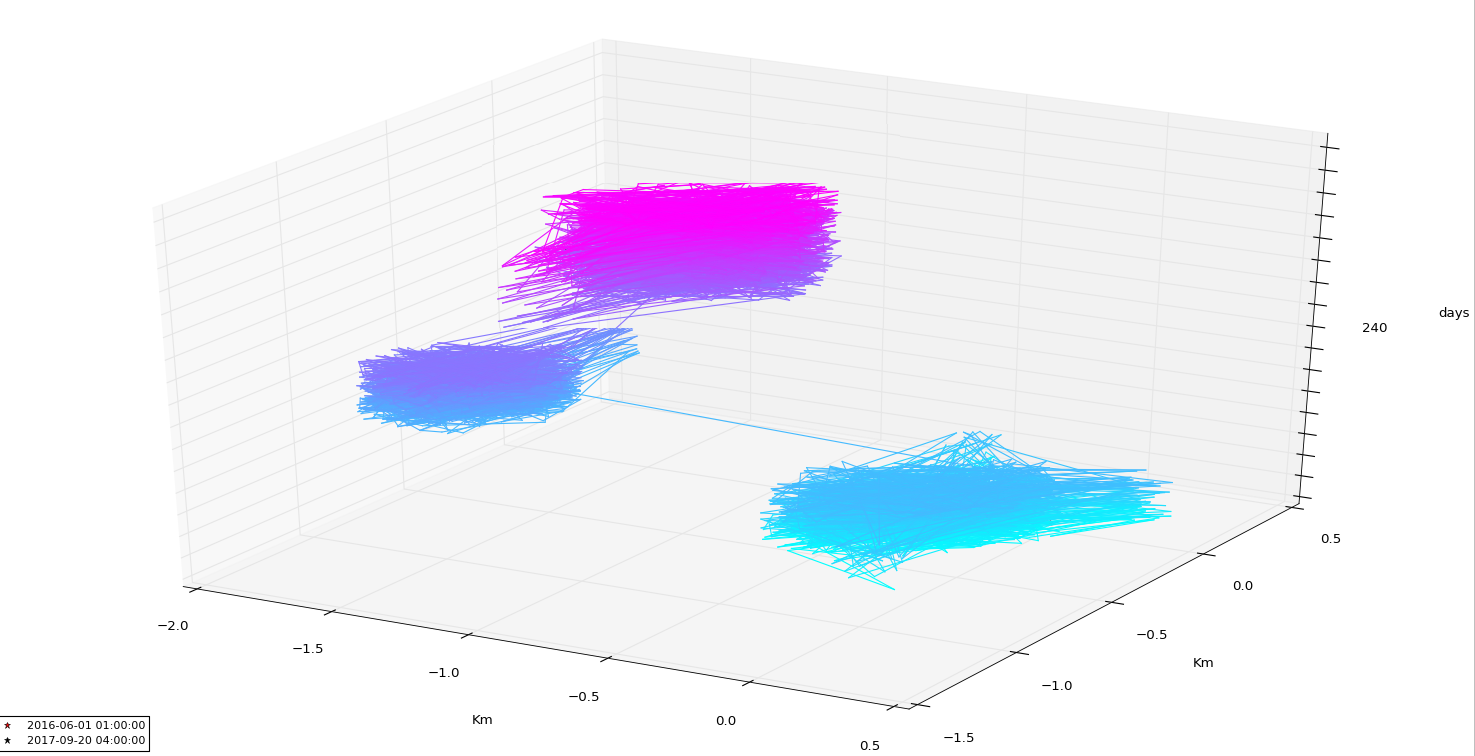}}}
	\subfigure[]	
	{\fbox{\includegraphics[height=3cm]{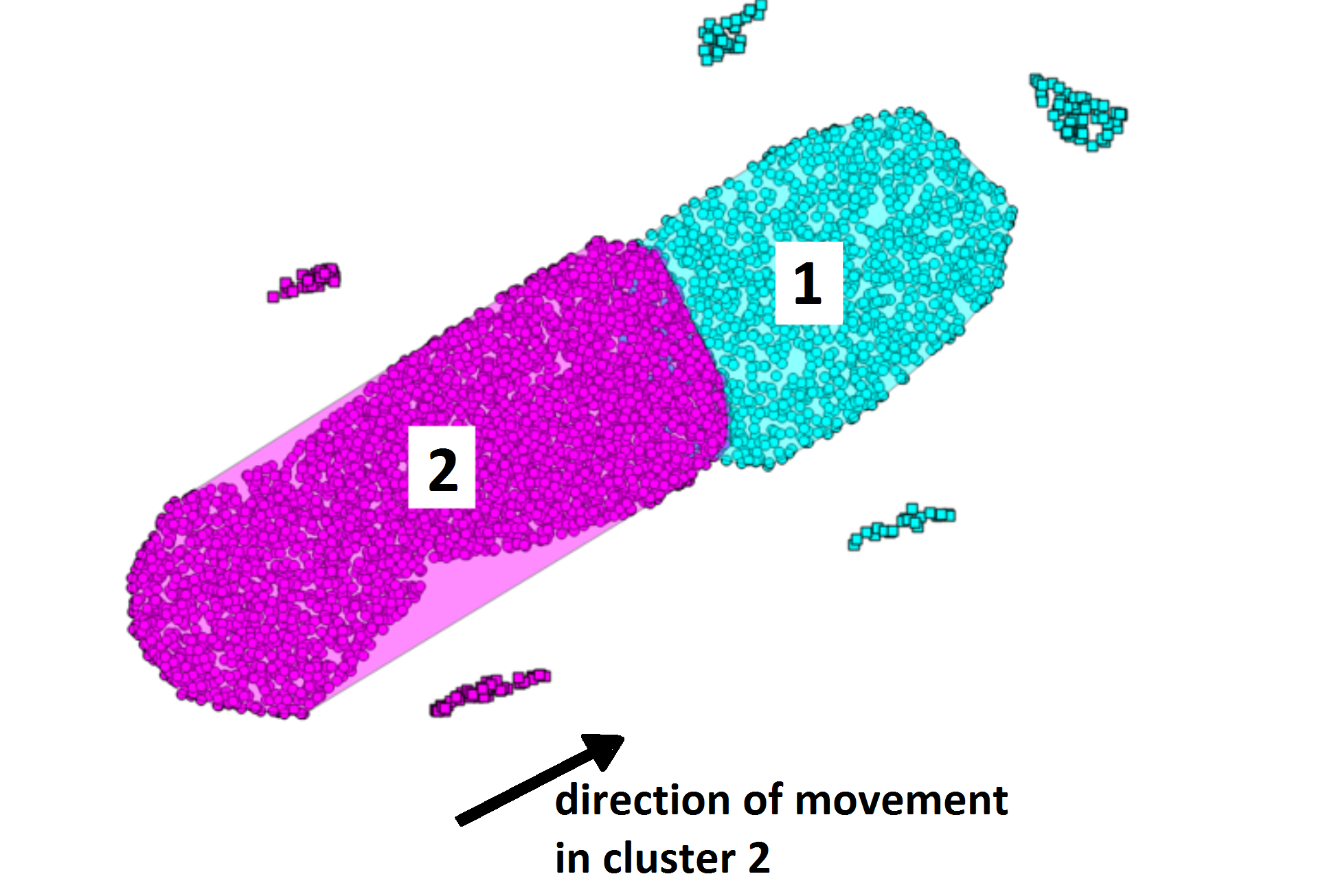}}}
	\subfigure[Pattern 3]
	{\fbox{\includegraphics[height=3cm]{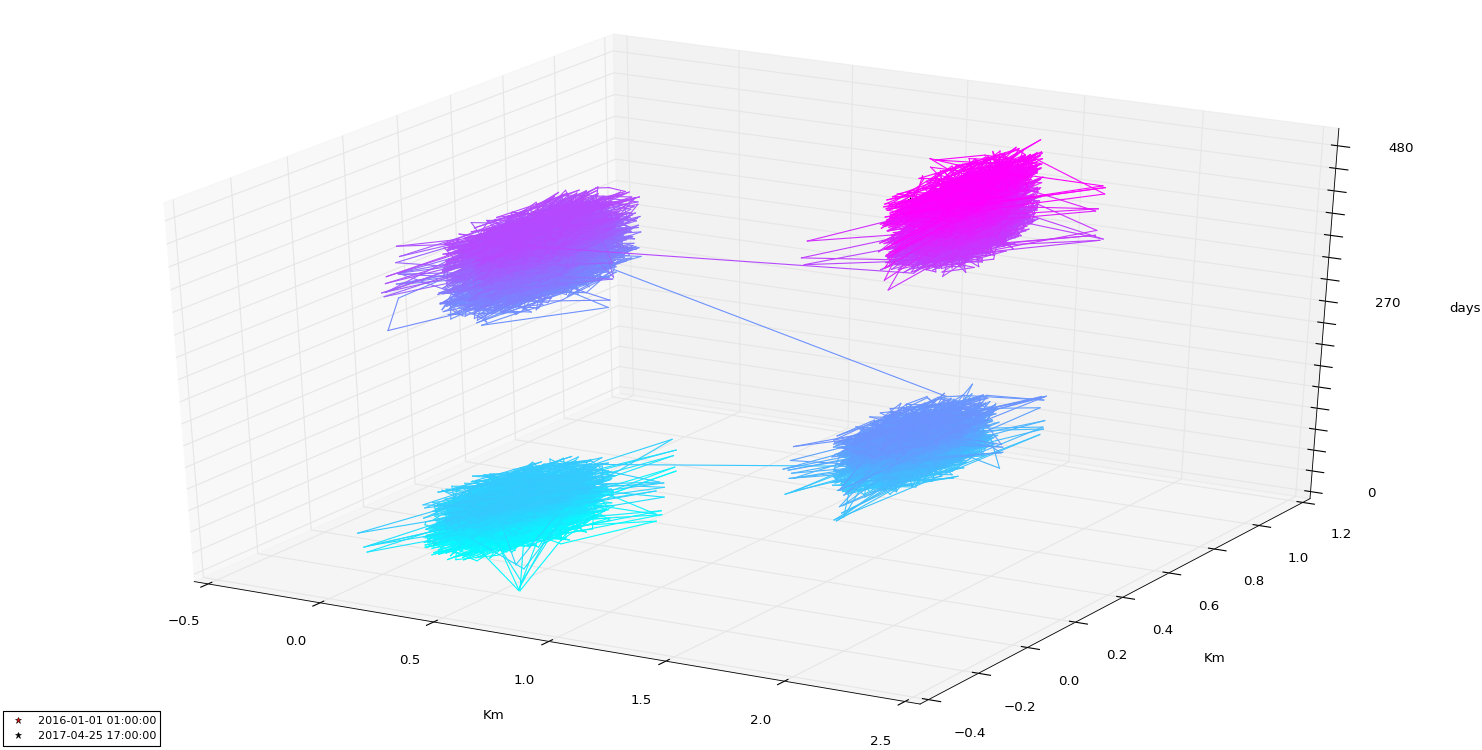}}}
	\subfigure[]	
	{\fbox{\includegraphics[height=3cm]{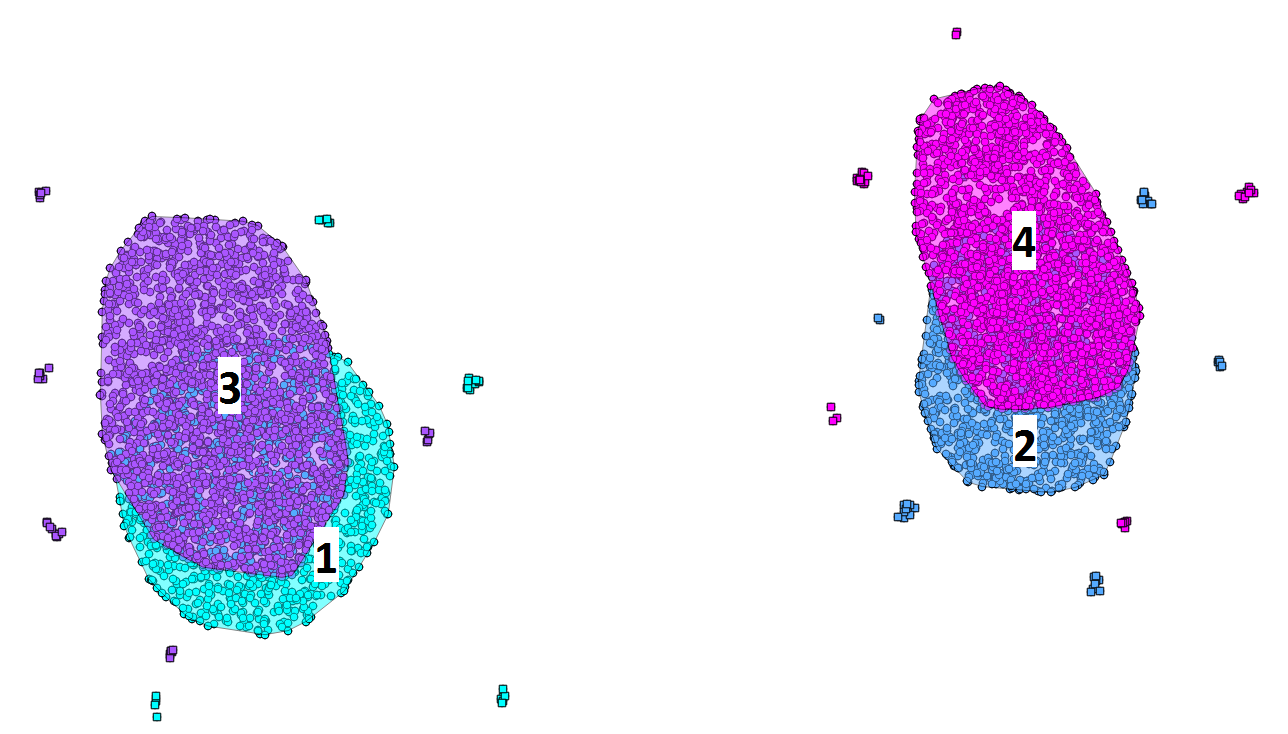}}}
	\caption{Patterns. (a,c,e) Spatio-temporal representation of the trajectory: the  vertical axis measures the temporal distance from the start of the trajectory (day unit), the color gradient the evolution in time, the space units the distance from the starting point. 
		(b,d,f) Segmentation: points are classified and displayed using a different symbology; the stay regions are enclosed in polygons to ease readability.} 
	\label{case1}	
\end{figure}

\section{Discovering derived patterns}

We argue that the SeqScan framework
can facilitate the discovery of additional mobility patterns. We refer to the patterns that are built on the notion of stay region as \emph{derived}.  
In this section, we present an approach to the discovery of recursive movement patterns \citep{Dataminingbook2014,Esa2015}. This type of movement can be broadly defined as \emph{repeated visitation to the same particular locations in a systematic manner} \citep{Esa2015}.
Detecting which and how those locations are frequented can reveal important features of the object behavior.
We focus, in particular, on the detection of locations that are  frequented regularly on a periodic basis. 
To avoid possible conflicts with the terminology used in the rest of the article, we call \emph{zones}  the 'locations' visited by an object. We split the problem in two sub-problems:  
\begin{itemize}
	\item To discover the  zones $z_i,..,z_j$ 
	\item To discover the  periodic zones, i.e., $Zone (t) = Zone(t + \mathcal{T})$, where $\mathcal{T}$ is the period and $Zone(t)$ the zone where the object is located at time $t$.
	
\end{itemize}
Coherently with the work presented so far, the overarching assumption is that the behavior may contain noise. 
\subsection{Discovery of zones}

Periodic zones are commonly modeled as spatial clusters \citep{Dataminingbook2014,Cao2007}.  
To extract these clusters,  Cao et al. use DBSCAN \citep{Cao2007}, while the MoveMine project \citep{Movemine2011} the Worton method \citep{Worton1989}. 
All of these clustering techniques ignore time, i.e. are \emph{spatial-only}.
\begin{figure}[H] 
	\center
	\subfigure[]{
		\includegraphics[height=4cm]{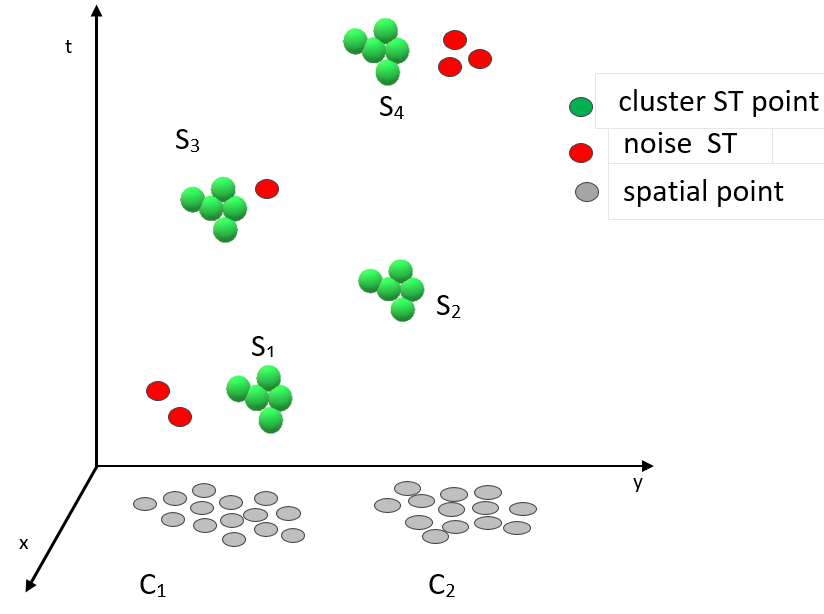}}
	\subfigure[]{
		\includegraphics[height=4cm]{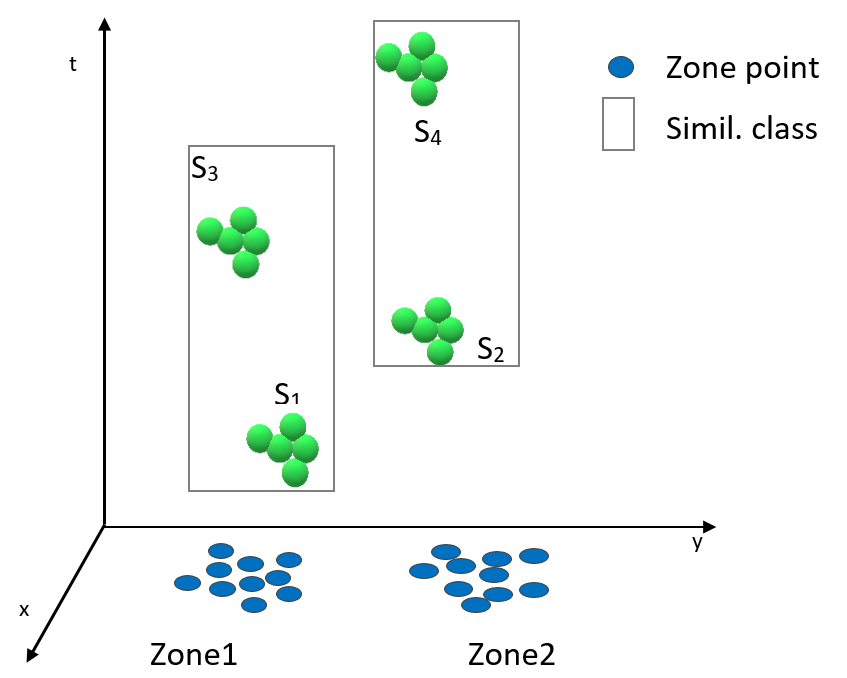}}
	\caption{Spatial-only vs. spatio-temporal clustering. (a) Spatial-only clusters as projection of  spatio-temporal clusters with noise; c) Spatio-temporal clusters grouped in similarity classes.}
	\label{palline}
\end{figure}
We argue that the use of spatial-only clustering, in place of spatio-temporal clustering, may result into a  rough approximation that impacts the quality of the analysis. 
To begin, we observe that  objects spend some time inside a zone.  Therefore, if the location is sampled at a frequency that is  relatively high with respect to the time spent inside a zone, a visit results in a dense set of sample points. While such a set can be straightforwardly modeled as a cluster, 
it is highly unlikely that the clusters at  different times are perfectly identical. 
That suggests modeling a frequented location as \emph{set} of clusters. 
This change of perspective has important implications. For example, it can be shown that spatial-only clustering  can generate clusters including points that in reality represent noise. An example can better explain the problem. 
Consider the points of a trajectory projected on  plane and assume that these points form the clusters $C_1, C_2$  as shown in Figure \ref{palline}.(a). These clusters appear compact, i.e., no noise.
In reality, there are  4  agglomerates, i.e. spatio-temporal clusters,  
along with a few noise points.  These agglomerates  are pairwise close to each other,  thus,
once projected on plane, collapse in a unique spatial cluster, which absorbs the noise points. As a result the noise information is lost.  Clearly, that can be avoided taking into account time. 
Another important reason for using spatio-temporal clusters, in particular stay regions, in place of spatial-only clusters, is that the individual movement can be given a discrete sequential representation, which can be more easily manipulated. 

In the light of these considerations, we propose the following approach: to extract the sequence of stay regions and then group together the stay regions that are close to each other, based on a properly defined notion  of proximity, to finally associate each such groups a \emph{zone}. We recall  that, based on Corollary \ref{cor2}, two  stay regions can \emph{overlap}. We need, however, a more restrictive and noise independent notion of cluster proximity, therefore we introduce the concept of \emph{spatial similarity} (of clusters). Similar stay regions form a \emph{zone}.  
In the following we detail the process starting from  the notions of  spatial similarity and zone.\\ 

\noindent
\textbf{Spatial similarity of clusters.}  
We say that two stay regions $S_1, S_2$ are spatially similar if there is at least one core point of $S_1$  that is \emph{directly reachable} from $S_2$ (in the DBSCAN sense), or viceversa, there is at least one core point of $S_2$  that is \emph{directly reachable} from $S_1$. 
We quantify the spatial similarity as the maximum percentage of core points that are directly reachable from the core points of the other set. 
More formally:
let $O_1 (O_2)$ be the set of core points in stay region $S_1 (S_2)$ falling in the $\epsilon$-neighborhood of some core point in $S_2$ ($S_1$).
We define the function of spatial similarity $Sim(S_1, S_2)$  as follows: 
\begin{equation}
Sim(S_1, S_2)=\max { \{\frac{|O_1|}{|S_1|},\frac{|O_2|}{|S_2|} \}}
\end{equation}
The two stay regions are spatially similar if
$Sim(S_1, S_2) \geq \psi$, where $\psi \in [0,1]$ is the similarity threshold.
Note that this notion of similarity is not affected by the relative size of clusters, in other terms the similarity can be 1, even though the clusters are of very different size.\\

\noindent
\textbf{Zones.} Similar stay regions can be grouped in classes. 
The \emph{similarity class} $C_i$ of the stay region  $S_i$ is defined recursively as follows: $C_i$ contains  $S_i$ and all of the regions similar to at least one region of the class. A similarity class is maximal, that is every stay region that can be added to the class,  belongs to the class. 
Moreover, the relation of  spatial similarity  induces a partition over the set of  stay regions. 
For every  similarity class we define the corresponding zone as follows. 
The  similarity class $C_i=\{S_i,..S_j\}$  associated with a zone $Z_i$ is the projection on space $\pi_{x,y}$ of the union set of the stay regions in $C_i$	
\begin{equation} 
Z_i= \pi_{x,y} (\bigcup_{S_j \in C_i} S_j)
\end{equation}
A nice property that follows from the above definitions is that a zone is itself a  cluster. However, in contrast with spatial-only clusters, zones do not contain noise.  An example is shown in Figure \ref{palline}.(b). The 4 stay regions $S_1 \rightarrow S_2 \rightarrow S_3 \rightarrow S_4$ of the Figure can be pairwise grouped to form 2 zones, \emph{Zone1} and \emph{Zone2}.  Note that, at this stage,  we can rule out the local noise because it is not relevant for the problem at hand. Thus the trajectory can be rewritten as sequence of temporally annotated zones  
for example using the formalism of symbolic trajectories \citep{tsas2015}: $(I_1, Zone1) (I_2, Zone2) (I_3, Zone1) (I_4, Zone2)$.
As a result, we obtain a simple and compact representation of the trajectory.

\subsection{Discovery of periodically visited zones}

The zones may be be visited periodically.  The period, however, is not  known, thus can range between  1 and $n/2$ where $n $ is the length of trajectory, moreover it can be imprecise.
To our knowledge the only approaches dealing with the periodicity of locations are built on spatial-only clustering \citep{Dataminingbook2014,Cao2007,Movemine2011}. 
For the analysis of location periodicity, we propose to leverage the symbolic representation of the trajectory obtained at the previous step, map it onto a time series and use a technique for the periodicity analysis of symbolic time series with noise. Specifically, we utilize the WARP technique  \citep{Elfeky2005}\footnote{The implementation of the algorithm has been kindly provided by M. Elfeky, co-author of WARP \citep{Elfeky2005}. Another implementation is available on: \url{//github.com/Serafim-End/periodicity-research}}.\\ 
\\
\noindent
\emph{WARP.} Consider a time series $T=[x_0,x_1,..,x_{n-1}]$  of $n$ elements. The key idea underlying WARP is that if we shift the time series of $p$ positions and compare the original time series to
the shifted version, we find that the time series are very similar, if $p$ is a candidate period \citep{Elfeky2005}. Therefore the greater the number of matching symbols, the greater the accuracy of the period. For  every possible period $p \in[1, n/2]$, WARP computes the similarity between the time series $T$ and the time series shifted $p$ positions, $T^{(p)}$ using Dynamic Time Warping (DTW) as similarity metric. The underlying distance function measures whether two symbols are identical or not. The value of the distance $DTW(T, T^{(p)})$ ranges between $0$ and $n-p$, where $0$
indicates that the time series of length $p$ is perfectly periodic.  
The confidence of a period $p$ is defined as: $ 1- \frac{DTW(T, T^{(p)})}{n-p}$. \\

\noindent
\emph{Process.}
We obtain the time series from the symbolic trajectory resulting from the previous phase as follows. First, we specify the temporal resolution of the time series, e.g. week, year. Next, for every temporally annotated zone, we create a sequence of repeating symbols, one per time unit. We repeat the same process with the \emph{transitions}, which are assigned a system-defined symbol.  The resulting time series is given  in input to WARP, which returns the candidate periods for every period $p$ along with the confidence value. The periods with high confidence are those of interest. 
An application of the method will be shown in the next section.

\section{Experimental evaluation of SeqScan}
\label{sec2}

We turn to discuss the process of SeqScan validation. The methodology consists of two phases: 
\begin{itemize}
	\item [(a)]
	External evaluation of SeqScan. We  confront the clustering with ground truth where the ground truth consists  of  synthetic trajectory data; 
	\item [(b)] Evaluation of the  technique for the discovery of periodic locations. We compare our approach with the solution developed in the MoveMine project
	\citep{Movemine2011}, based on a real dataset.
\end{itemize}  
For the efficiency aspects, we refer the reader to 
earlier work \citep{Damiani:2014}. 
The experiments are conducted using
the MigrO environment, a plug-in
written in Python for the Quantum GIS
system\footnote{http://www.qgis.org} providing a number of
functionalities for SeqScan-based analysis, 
including visualization tools \citep{Damiani2015}.
For the statistical tests, we use the R system\footnote{https://www.r-project.org/}. The hardware platform consists of 
a computer equipped with Intel i7-4700MQ, 2.40GHz processor with 8 GBytes of main memory.


\subsection{Part 1: external evaluation of SeqScan}

In general, the ground truth consists of labeled points  where the labels specify the categories the points belong to. 
The ground truth can be either generated by a simulator  
or consist of  real data labeled by domain experts. In both cases,  ensuring a fair evaluation may be problematic: real trajectory data is commonly of low quality (e.g. missing points), while 
synthetic datasets can be engineered to match assumptions
of the occurrences and properties of meaningful clusters \citep{Kriegel2010}. For a sustainable and fair evaluation, we present a different methodology.  
We still use synthetic data, nevertheless 
the trajectories are
generated  by a simulator 
grounded on an independent model designed for the simulation of the animal movement. Moreover, 
the comparison is performed using blind experiments 
(see also \citep{Cagnacci2015} for a similar approach). The experimental setting, i.e. the synthetic dataset and the evaluation metrics used for the experiments,  is detailed in the following.\\

\begin{figure}[H]
	\begin{center}
		\subfigure[Trajectory IND1]
		{\includegraphics[width=4.5cm]{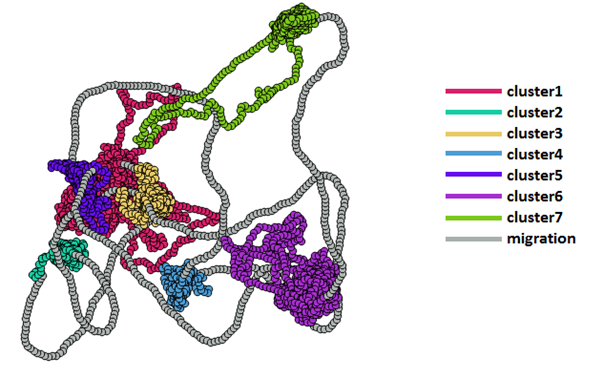}}
		\subfigure[]
		{\includegraphics[width=5.5cm]{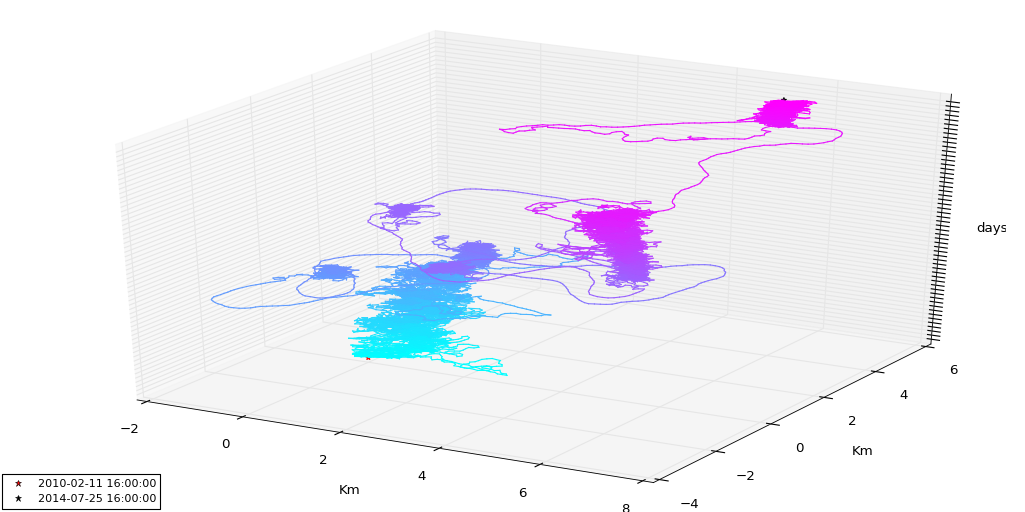}}
		\subfigure[Trajectory IND14]
		{\includegraphics[width=5.5cm]{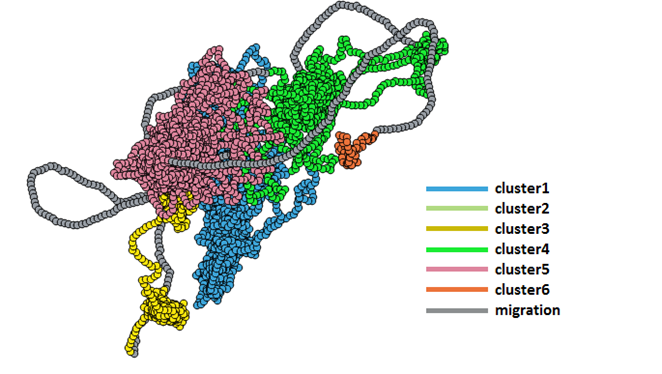}}
		\subfigure[]
		{\includegraphics[width=5cm]{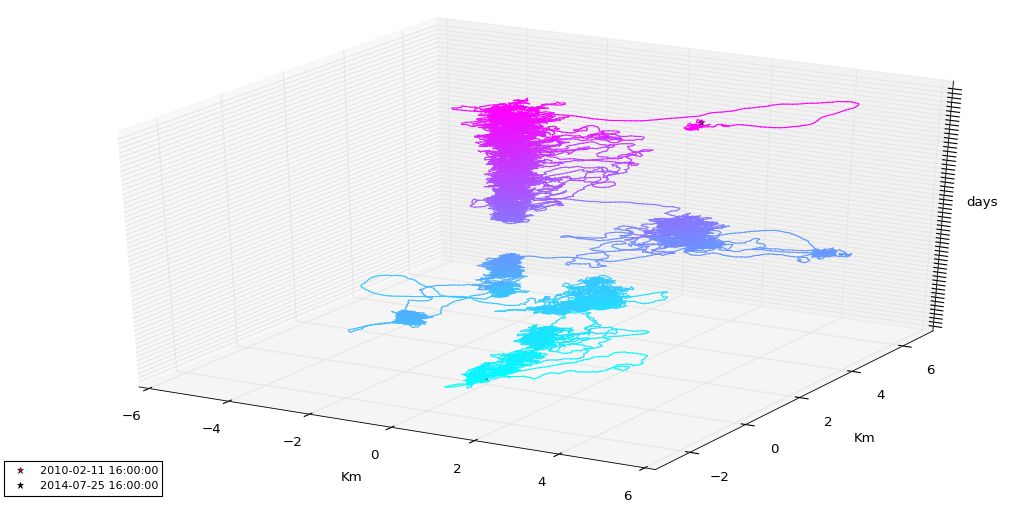}}
	\end{center}
	\caption{Examples of synthetic trajectories: (a) planar representation, (b) spatio-temporal representation. The vertical axis is for time (days).} 
	\label{groundtruth}		
\end{figure}
\noindent
\textbf{Synthetic data.} Animal trajectories are simulated as a stochastic movement process ({\it sensu}\citep{Pat1953}) in which animals move on a landscape interspersed by randomly
distributed resource patches. The stochastic movement process is for
the animal follows that of  Van Moorter et al. \citep{Van2009}, in which animals move with
in a fixed step length, with their movement directions biased towards
resource patches and (specifically with the bias was proportional to
patch attraction values, where closer resource patches to the animal’s
position were more attractive to account for the costs of
movement). In addition, following Van Moorter et al. \citep{Van2009}
each animal had a two-component spatial memory that reinforces attraction towards previously visited areas. This attraction feedback allows the movement model to capture the spatially-localized nature of movement behavior observed in many empirical animal trajectories (see \citep{Van2009} for further details). 
In order to provide sufficiently complex trajectories for the evaluation of SeqScan, we generated several behavioral modes of movement \citep{Mor2004} – namely residence, excursion  and migration  – by altering the relationship between animal movement and the food resources. Each movement mode is the result of an underlying set of stochastic movement rules that generates certain realized characteristic spatial pattern of animal relocations.

The ground truth is finally extracted from the simulated trajectories by mapping the ecological concepts  onto the concepts of our model.
The process has been applied to create a  dataset of 12 spatial  trajectories of 19,500 points each, with labeled points indicating stay region, transition, local noise, and time interval of  2 hours. This dataset, called hereinafter 'animal dataset', is the ground truth.   Two examples trajectories are shown in Figure \ref{groundtruth}. 
The full set of trajectories is reported in Appendix while summary statistics on the distribution of points and the number of clusters in each trajectory are reported in Table \ref{tab:stat-12}.

\begin{table}
	\centering
	\small{
		\caption{Point classification in  the ground truth}
		\label{tab:stat-12}
		\begin{tabular}{|>{\bfseries}L{1cm}|>{\bfseries}C{2cm}|C{2cm}|C{2cm}|C{2cm}|}
			\hline
			Traj-Id & number of clusters & \textbf{\%  clustered points} & \textbf{\% local noise } & \textbf{\% transition points} \\
			\hline
			Ind6 & 7 & 87.06 & 10.49 & 2.45 \\
			\hline
			Ind41& 7 & 91.09 & 6.50 & 2.41 \\
			\hline
			Ind1& 7 & 92.06 & 4.80 & 3.14 \\
			\hline
			Ind35 & 6 & 85.27 & 12.44 & 2.28 \\
			\hline
			Ind10 & 6 & 85.85 & 12.14 & 2.00 \\
			\hline
			Ind39 & 6 & 86.45 & 11.57 & 1.97 \\
			\hline
			Ind14 & 6 & 86.57 & 11.54 & 1.88 \\
			\hline
			Ind25  & 6 & 86.23 & 11.36 & 2.40 \\
			\hline
			Ind17 & 6 & 86.61 & 11.26 & 2.13 \\
			\hline
			Ind12 & 6 & 87.41 & 10.12 &2.47 \\
			\hline
			Ind8 & 6 & 88.17 & 9.70 & 2.12 \\
			\hline
			Ind49 & 6 & 88.58 & 8.88 &  2.54 \\
			\hline
		\end{tabular}
	}
\end{table}
\paragraph{\emph{\textbf{\textbf{Evaluation metrics.}}}}
%
For deliberate redundancy, we choose two metrics from different families, set-matching and counting pairs, respectively \citep{basu2004active}.
Further, we consider a third  metric counting the different number of clusters as  simple measure of structural  similarity of segmentations.   Let  $R = \{r_i\}_{i \in [1,n]}$  be the set of $n$ 'true' clusters, 
$S = \{s_i\}_{i\in [1,m]}$ the set of $m$ stay regions detected by SeqScan, $N_s$ the total number of clustered points in the SeqScan output, and $N_r$ the total  number of clustered points in the ground truth. The metrics are described in the sequel  while their definition is reported in Table \ref{tab:eval}:

\begin{table}[H]
	\centering
	\caption{Evaluation metrics}
	
	\begin{tabular}{|l|L{7cm}|}
		\hline
		Metric & Defs\\
		\hline
		\multirow{4}{*}{H-Purity (R, S)} &Purity(R,S)= $\frac{1}{N_s} \sum_k \max_j |s_k \bigcap r_j|$\\
		&InvPurity(R,S)= $\frac{1}{N_r} \sum_k \max_j |s_j \bigcap r_k|$\\
		&H-Purity (R,S)= $\frac{2 \times Purity(R,S) \times InvPurity(R,S)}{Purity(R,S)+InvPurity(R,S)}$\\
		\hline
		\multirow{6}{*}{Pairwise F-measure (R,S)} & TP=\#pairs assigned to the  same cluster is R and S\\
		& FP=\#pairs assigned to different clusters  in R but to the same cluster in S\\
		& FN=\#pairs assigned to different clusters in S but to the same cluster in R\\
		& Precision=$\frac{TP}{TP+FP}$\\
		& Recall=$\frac{TP}{TP+FN}$\\
		& F-measure (R,S)=$\frac{2 \times Precision (R,S) \times Recall(R,S)}{Precision(R,S)+Recall(R,S)}$\\
		\hline
		Diff (R,S) & $|card(S)-card(R)|$\\
		\hline
	\end{tabular}
	\label{tab:eval}
\end{table}

\begin{itemize}
	\item [i)]  Harmonic mean of Purity 
	and Inverse Purity \citep{amigo2009comparison} 
	hereinafter denoted H-Purity. In general, 
	the Purity metric penalizes clusters containing items from  different categories, while it does not reward the grouping of items from the
	same category. By contrast, Inverse Purity rewards grouping items together, but it does not penalize mixing items from different categories. 
	The H-Purity metric mediates between Purity and Inverse Purity.  
	\item [ii)]  Pairwise F-measure. This metric represents the harmonic mean of pairwise precision and recall computed over the contigency table specifying the number of pairs that are correctly/incorrectly  classified as members of an identical/different cluster in S and R, respectively \citep{basu2004active,amigo2009comparison}. 
	\item [iii)] The third  metric,  \emph{Diff} in brief, computes the difference  between the number of stay regions and the number of clusters
	
\end{itemize}


\subsubsection{Experiments}
We present a series of 5 experiments focusing on the following two aspects: 
\begin{itemize}
	\item [-] Quantitative evaluation: we present a systematic approach to the \emph{quantitative} evaluation of the cluster-based segmentation against ground truth. 
	\item [-] Parameter sensitivity: 
	we analyze the sensitivity of SeqScan to key internal and external parameters.
\end{itemize} 
\textbf{Experiment 1: structural comparison and impact of the $\delta$ parameter. }

For every trajectory of the dataset,  SeqScan is run with parameters that only differ for the value of the presence  $\delta$, set to 20 days and 100 days respectively. In both cases the density parameters are: $\epsilon=200, K=50$. Such parameters are chosen through an iterative process. 
The resulting number of clusters  contrasted with the true number  is reported in Table \ref{tab:nb-of-clusters}.  
With $\delta=20$ days  the  number of stay regions in the trajectory segmentation  is substantially close to the number of  clusters in the ground truth while with $\delta=100$ days the number of stay regions is substantially different (the statistical significance is evaluated through the Kruskal-Wallis test \citep{Gib2003}). 
Visual analysis can provide further information.  

\begin{figure}
	\begin{center}
		\subfigure[IND1: $\delta=20$]
		{\includegraphics[height=4.5cm]{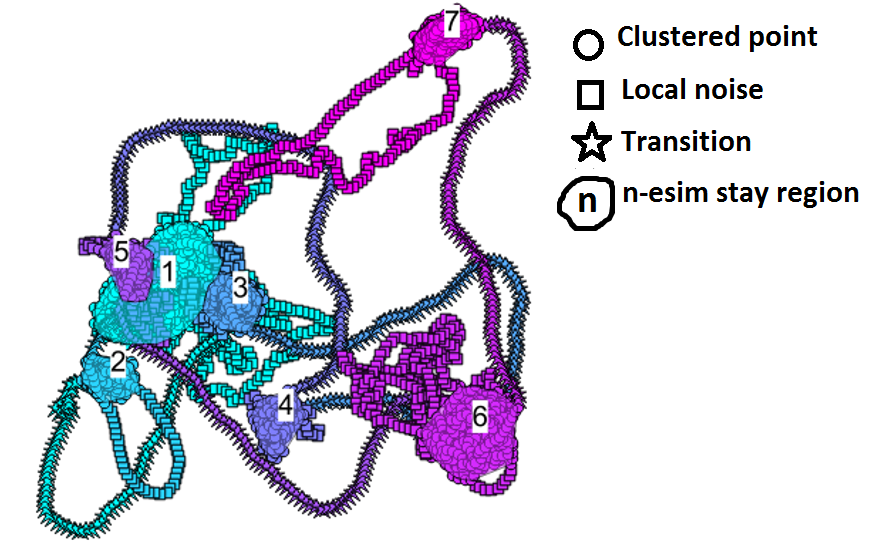}}
		\subfigure[IND1: $\delta=100$]
		{\includegraphics[height=4.5cm]{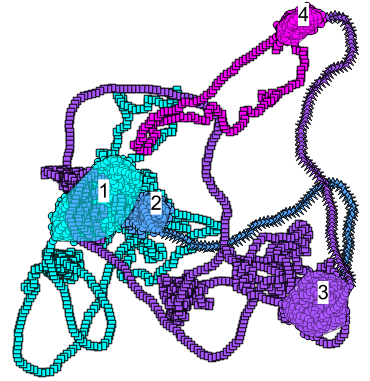}}
		\subfigure[IND14: $\delta=20$]
		{\includegraphics[height=4.5cm]{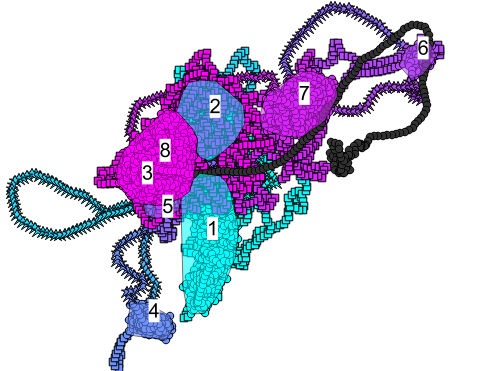}}
		\subfigure[IND14:$\delta=100$]
		{\includegraphics[height=4.5cm]{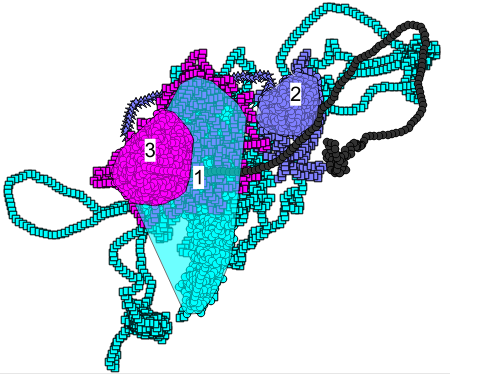}}		
	\end{center}
	\caption{Segmentation of the trajectories IND1 and IND 14 for different values of the parameter $\delta$: (a,c) $\delta=20$; (b,d) $\delta=100$. The clustered points are enclosed in a progressively numbered polygon obtained as convex hull of the set of points. }  
	\label{fig:ind1}		
\end{figure}

As an example,  Figure \ref{fig:ind1} 
illustrates the segmentation of the trajectories IND1 and IND14. 
The segmentation of the  trajectory IND1 in Figure \ref{fig:ind1}.(a) for $\delta=20$ days,  contains the same number of clusters of the ground truth in Figure \ref{groundtruth}. Moreover there is  visual evidence of good matching of the segmentation with the ground truth.  By contrast, the segmentation in Figure \ref{fig:ind1}.(b) only contains four large clusters. In this sense, the parameter $\delta$ allows for the tuning  of the temporal granularity of clusters.  In the second trajectory, IND14, the number of clusters detected by SeqScan  with $\delta=20$ days is 8 against the true 6 clusters. 
It can be noticed that SeqScan recognizes as distinct two  clusters (cluster 1 and 2) that in reality are part of a unique cluster in the ground truth.  With $\delta=100$ days, the number of clusters decreases to 3. 

\begin{table}[H]
	\centering
	\small{
		\caption{Experiment 1: Number of clusters with $\delta=20$  and $\delta=100$}
		\begin{tabular}{|C{3cm}|C{2cm}|C{2cm}|C{2cm}|}
			\hline
			\multirow{2}{*}{\textbf{Trajectory-Id}} &  \multicolumn{3}{c|}{ \textbf{Number of clusters}}\\
			\hhline{~---}
			& \textbf{Ground truth} & \textbf{SeqScan: $\delta$=20days} & \textbf{SeqScan: $\delta$=100days} \\
			\hline
			\textbf{IND1} & 7 & 7 & 4 \\
			\hline
			\textbf{IND6} & 7 & 5 & 3 \\
			\hline
			\textbf{IND8} & 6 & 7 & 5 \\
			\hline
			\textbf{IND10} & 6 & 6 & 2 \\
			\hline
			\textbf{IND12} & 6 & 4 & 3 \\
			\hline
			\textbf{IND14} & 6 & 8 & 3 \\
			\hline
			\textbf{IND17} & 6 & 7 & 3 \\
			\hline
			\textbf{IND25} & 6 & 6 & 4 \\
			\hline
			\textbf{IND35} & 6 & 8 & 5 \\
			\hline
			\textbf{IND39} & 6 & 8 & 3 \\
			\hline
			\textbf{IND41} & 7 & 6 & 4 \\
			\hline
			\textbf{IND49} & 6 & 7 & 1 \\
			\hline
		\end{tabular}
		\label{tab:nb-of-clusters}
	}
\end{table}

\paragraph{\textbf{Experiment 2: analysis of the quality indexes.}} 
%
The visual comparison performed at the previous step, though useful, does not provide any quantitative measure. To that end,  we run SeqScan with the 'good'  parameters determined at the previous step: $\epsilon=200, K=50, \delta=20$ and compute, for every trajectory, the quality indexes resulting from the comparison of the SeqScan outcome with the ground truth.   
Table \ref{tab:h-purity-diff-presence} reports the indexes H-Purity and Pairwise F-Measure for every trajectory. 
It can be seen that the two indexes are substantially aligned and that overall the quality of the clustering is high. It can be noticed however that the H-purity value of the trajectory IND14  seen earlier - the one reporting the highest structural difference -  is among the lowest in the table. This means that  clusters may lack either compactness or homogeneity, in line with the visual analysis. By contrast the quality of the segmentation of IND1 is quite high, again in line with the structural comparison.  
Overall, the indices show  a good matching  with the ground truth both at the level of structure  and of single clusters. 
\begin{table}
	\centering
	\small{
		\caption{Experiment 2:  H-Purity and Pairwise F-measure}
		\label{tab:h-purity-diff-presence}
		\begin{tabular}{|c|c|c|}
			\hline
			\textbf{Traj-Id} & \textbf{H-purity} & \textbf{Pairwise F-measure}\\
			\hline
			\textbf{IND1} & 0.98 & 0.95 \\
			\hline
			\textbf{IND6} & 0.89 & 0.73  \\
			\hline
			\textbf{IND8} & 0.90  & 0.84\\
			\hline
			\textbf{IND10} & 0.93 & 0.86 \\
			\hline
			\textbf{IND12} & 0.93 & 0.83 \\
			\hline
			\textbf{IND14} & 0.87 & 0.84\\
			\hline
			\textbf{IND17} & 0.95 & 0.89\\
			\hline
			\textbf{IND25} & 0.9 & 0.81\\
			\hline
			\textbf{IND35} & 0.91 & 0.83 \\
			\hline
			\textbf{IND39} & 0.91 & 0.85\\
			\hline
			\textbf{IND41} & 0.97 & 0.93\\
			\hline
			\textbf{IND49} & 0.96 & 0.92\\
			\hline
		\end{tabular}
	}
\end{table} 
\hspace{-0.7cm}
\paragraph{\textbf{Experiment 3: noise analysis.}}In this experiment we analyze the  contribution of the local noise  to the  quality of clustering. We recall that the unclustered points  can represent either transitions or local noise.
For this experiment, we contrast the quality of clustering  in presence of  local noise with the quality of the clustering in absence of local noise  (i.e. the local noise points are seen as elements of the clusters).  For the evaluation, we use the Pairwise F-measure because seemingly less favorable  than H-Purity. Table \ref{tab:F-measure} reports the value of the index in the two cases. The Kruskal-Wallis test  confirms the significance of the discrepancy that can be seen in the table, or, put in other terms, that the local noise has an impact on the quality of clustering.\\ 

\begin{table}
	\centering
	\caption{ Experiment 3. Pairwise F-measures in the two cases: clustering with no local noise  and clustering with  local noise, respectively. The green color highlights the greater value in each row}
	\label{tab:F-measure}
	\small{
		\begin{tabular}{|c|C{4cm}|C{4cm}|}
			\hline
			\multirow{2}{*}{\textbf{Traj-Id}} &  \multicolumn{2}{c|}{ \textbf{Pairwise F-measure}}\\
			\hhline{~--}
			&\textbf{no local noise} & \textbf{with local noise}\\
			\hline
			\textbf{IND1} & \colorbox{green!15}{0.99} & 0.95  \\
			\hline
			\textbf{IND6} & \colorbox{green!15}{0.81} & 0.73  \\
			\hline
			\textbf{IND8} & \colorbox{green!15}{0.98} & 0.84 \\
			\hline
			\textbf{IND10} & \colorbox{green!15}{0.97} & 0.86  \\
			\hline
			\textbf{IND12} & \colorbox{green!15}{0.92} & 0.83  \\
			\hline
			\textbf{IND14} & \colorbox{green!15}{0.93} & 0.84  \\
			\hline
			\textbf{IND17} & \colorbox{green!15}{0.98} & 0.89  \\
			\hline
			\textbf{IND25} & \colorbox{green!15}{0.90} & 0.81  \\
			\hline
			\textbf{IND35} & \colorbox{green!15}{0.96} & 0.83   \\
			\hline
			\textbf{IND39} & \colorbox{green!15}{0.96} & 0.85 \\
			\hline
			\textbf{IND41} & \colorbox{green!15}{0.99} & 0.93   \\
			\hline
			\textbf{IND49} & \colorbox{green!15}{0.97} & 0.92  \\
			\hline
		\end{tabular}
	}
\end{table}

\noindent
\textbf{Experiment 4: analysis of the $\delta$ parameter.}
In this experiment, we analyze the behavior of the function $f_T()$  describing the relationship between the  presence parameter $\delta$ and the number of stay regions detected by SeqScan. We recall that  the function $f_T()$ is computed by the Algorithm \ref{algo:f}  that repeatedly runs SeqScan with different values of $\delta$.  This function is implemented as part of the MigrO environment. 
In this experiment, the function is built by  running SeqScan with the usual density parameters  $\epsilon= 200$ and $K=50$ and with $\delta$ varying between 0 and the duration of the trajectory. The number of iterations  is limited to 160, i.e. SeqScan is invoked 160 times.  Figure \ref{fig:graph_ind14_10days_x160}  displays the plots of $f_T()$ for the  trajectory
IND1. 
It can be seen that 
the maximum number of clusters is obtained with $\delta=0$. 
In particular,  the segmentation of IND1  (Figure \ref{fig:graph_ind14_10days_x160}) consists of 7 clusters for $\delta$ ranging between 0 and 32 days. This is in line with the result in Table \ref{tab:nb-of-clusters} reporting the number of clusters for $\delta=20$ days. 

\begin{figure}
	\centering
	\subfigure[Traj IND1]
	{\includegraphics[width=10cm]{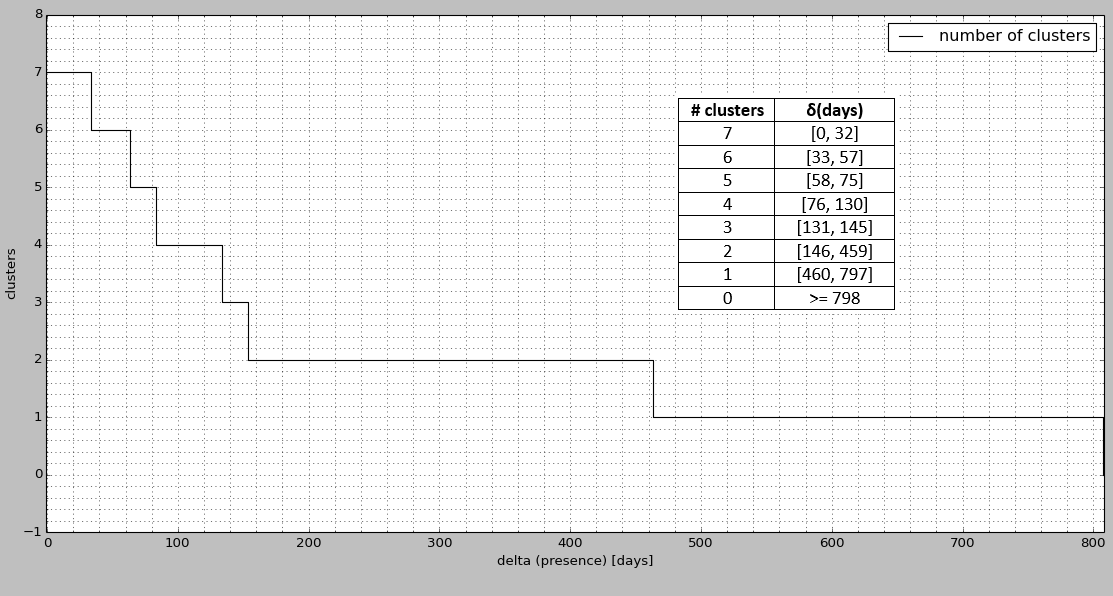}}
	\caption{Experiment 4. The function $f_T$  for the trajectory IND1. The function is reported in both graphic and tabular form.} 
	\label{fig:graph_ind14_10days_x160}
\end{figure} 

\begin{figure}[H]
	\centering
	\subfigure[Difference in $\#$clusters]
	{\fbox{\includegraphics[width=5.5cm]{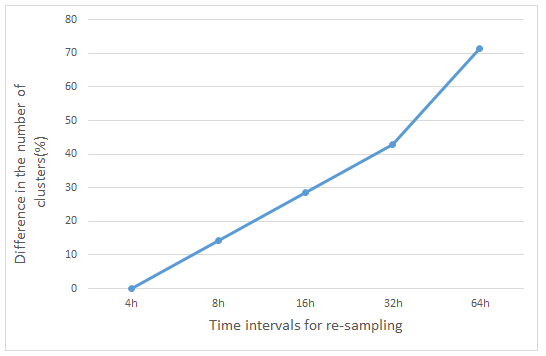}}}
	\subfigure[ Pairwise F-measure]
	{\fbox{\includegraphics[width=5.5cm]{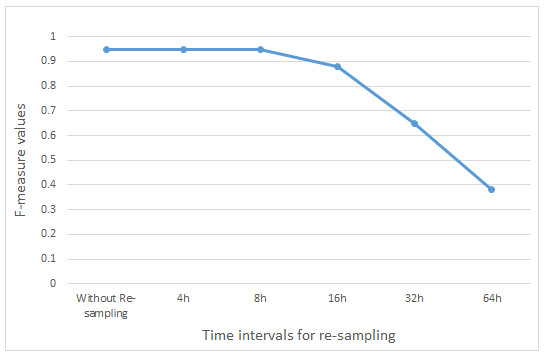}}}
	\caption{Experiment 5. Re-sampling of the trajectory IND1 (a) Normalized difference in the number of clusters (with respect to the regular trajectory); (b)
		Pairwise F-Measure computed w.r.t. ground truth} 
	\label{last}
\end{figure}

\paragraph{\emph{\textbf{Experiment 5: sensitivity to the sampling rate. }}}
For this experiment, the trajectories are re-sampled considering time intervals of 4, 8, 16, 32, 64  hours (we recall that in the animal dataset the time interval has a width of 2 hours). Next, SeqScan is run over the re-sampled trajectories and the result contrasted with the ground truth, using two of the quality indexes discussed earlier, i.e., the difference in the number of clusters with respect to the regular trajectory (normalized) and the Pairwise F-measure.  SeqScan is run with clustering parameters  $\epsilon$=200, $\delta$=20 and $K$=50 points. The corresponding graphs for one of the trajectories are reported in Figure \ref{last}. 
It can be seen that for lower sampling rates, the normalized difference in the number of clusters increases, while the quality of the cluster (i.e. F-measure) decreases. While this is the expected behavior, more interesting is the fact that the quality of clustering is not dramatically compromised if the sampling rate is reduced by 2 and even 4 times (i.e. time interval of 4 and  8 hours).  This trend can be observed also in the other trajectories. 

\subsection{Part 2: evaluation of the SeqScan-based technique for the discovery of periodic locations}
We turn to use SeqScan with real data and confront 
the solution proposed for the discovery of periodic locations with the MoveMine approach \citep{Movemine2011}. 
We use a dataset containing the GPS trajectory of one eagle observed for nearly three years while flying between US and Canada. The sample points have been collected  between mid January 2006 and end December 2008 at a sampling rate that is highly irregular. The data can be downloaded from the Movebank database \footnote{\url{//http:www.movebank.org}. Study: Raptor Tracking: NYSDEC}. Notably, this dataset is the same used in the MoveMine project. Some cleaning operations are  preliminarily performed over data. As a  result, we obtain a trajectory of 14,442 points, extending over 1080 days, from 2006 Jan 15 until 2008 Dec 30
with an average step length of 3210 meters. 
The trajectory and its spatio-temporal representation is reported in Figure \ref{exp:1}.(a) and \ref{exp:1}.(c), respectively.
%
In the following, we analyze: (a) the zones, (b) the periodicity of zones. \\
\\
\noindent
\textbf{Zones discovery.} 
The analysis is performed in three main steps: \\
\noindent
\emph{Step 1}. Compute the sequence of stay regions. We run SeqScan with parameters: $\epsilon =60 km, N=100 points, \delta=20 days$. 
We obtain 12 stay regions (numbered from 1 to 12). The stationarity index \citep{Damiani2016} 
is generally high, meaning that the local noise in the region is limited and thus the staying is 'temporally dense'. 
We recall that the  transitions and the local noise are not relevant for this kind of analysis.

%
\begin{figure} [t]
	\center
	\subfigure[The trajectory]
	{\fbox{\includegraphics[height=6cm]{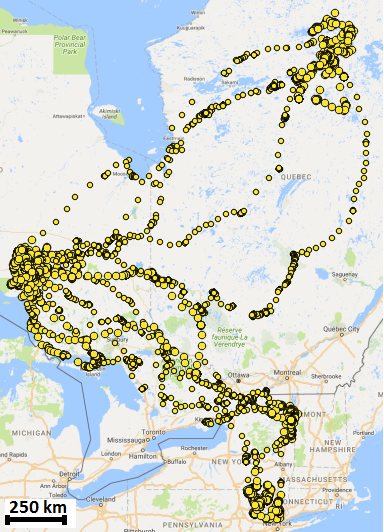}}}
	\subfigure[Stay regions 1-12 and zones]
	{\fbox{\includegraphics[height=6cm]{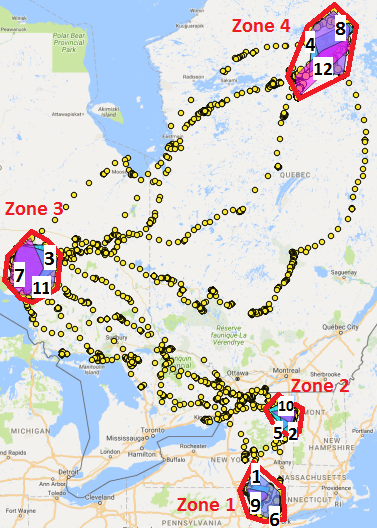}}}
	\subfigure[Spatio-temporal representation]
	{\fbox{\includegraphics[height=3.3cm]{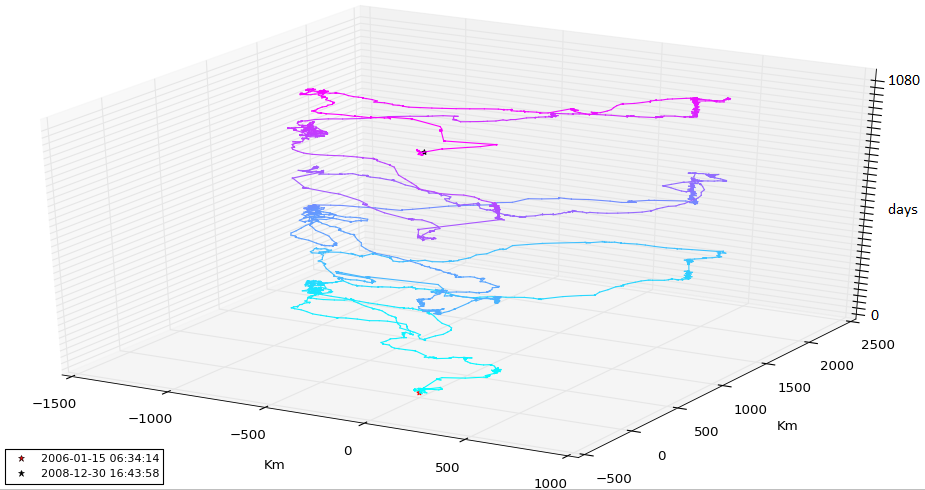}}}
	\subfigure[Spatial similarity table]
	{\fbox{\includegraphics[height=3.3cm]{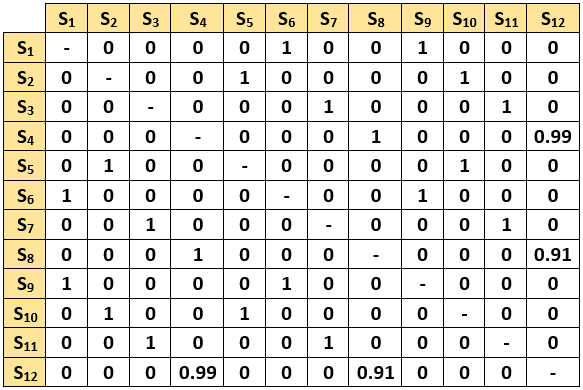}}}
	\caption{Analyzing the real trajectory of an eagle: stay regions and zones discovery.}
	\label{exp:1}	
\end{figure}
\noindent
\emph{Step 2}. Compute the similarity classes. We set the parameter $\psi=0$ and obtain four classes, each containing three stay regions: $C_1=\{1,9,6\}, C_2=\{2,5,10\},C_3=\{3,7,10\}, C_4=\{4,8,12\}$. The similarity table is reported in Figure \ref{exp:1}.(d).

\noindent
\emph{Step 3}. Compute the zones and the number of visits.
For each class we compute the corresponding zone as union set of the stay regions. The zones are denoted:  $\hat{1},\hat{2},\hat{3},\hat{4}$. 
The sequence of 12 stay regions can be rewritten in terms of zones: 
$\hat{1}, \hat{2}, \hat{3}, \hat{4}, \hat{2}, \hat{1}, \hat{3}, \hat{4}, \hat{1},  \hat{2},\hat{3}, \hat{4}$. The stay regions and the grouping in zones  are reported in Figure \ref{exp:1}.(b). 
Every zone is visited three times.
It is evident that the sub-sequence $\hat{1}, \hat{2}, \hat{3}, \hat{4}$ repeats itself, although with some irregularity. The following periodicity analysis provides further information.\\

\noindent
\textbf{Periodicity analysis.}
We analyze first the periodicity of single zones and then of the entire trajectory. The temporal resolution of time series is set to 1 week.
For each zone, we create a time series as follows. We consider the trajectory in the period between the beginning of the first visit and the end of the last visit. We split the temporal extent of the trajectory in weeks. Hence  for every week, if the object is inside the zone at any instant, we create the symbol '1', '0' otherwise. We recall that the local noise is overlooked at this stage, thus the object is assumed be continuously present inside a stay region. We run the WARP algorithm  over the time series and we select the smallest period with the highest confidence value.  
As a result, we obtain: two zones ($\hat{3}, \hat{4})$ have periods 53 and 51 weeks, respectively, with maximum confidence; the other two zones $\hat{1}, \hat{2}$ have periods 48 and 55 weeks, respectively. The  confidence of the period of zone $\hat{2}$, is the lowest (0.9) among the zones. We recall that Warp is built on the DTW distance which allows
an elastic shifting of the time axis to accommodate similar,
yet out-of-phase, segments \citep{Elfeky2005}, therefore,  a period $p$ can have maximum confidence although $p$ is not perfect.  
As we will explain in a while, we consider the set $\{\hat{1},\hat{2}\}$ as a unique zone. From the analysis of the time series created out of the sequence of zones and transitions using an appropriate number of symbols, we find that the period of the behavior is 52 weeks with maximum confidence.   Figure \ref{exp:period} illustrates the temporal sequence of zones traversed during the 3-years travel.   
\begin{figure} [H]
	\center
	\fbox{\includegraphics[width=7cm]{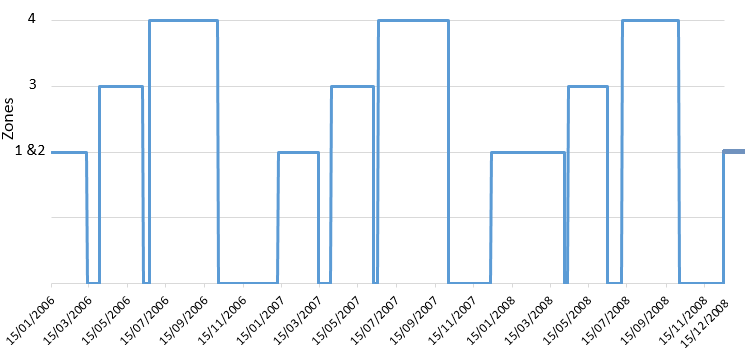}}
	\caption{The temporal sequence of zones traversed during the flight beginning on mid January 2006. The zones $\{\widehat{1}, \widehat{2}\}, \widehat{3}, \widehat{4}$ are indicated on the y-axis, time on the x-axis, valleys correspond to transitions.} 
	\label{exp:period}	
\end{figure}
%

\noindent
\textbf{Comparison.}
%
The map generated by MoveMine
(Figure 9 in \citep{Movemine2011}) highlights 3 \emph{reference spots} in the areas of New York, Great Lakes and Quebec, respectively. 
\begin{table}[H]
	\centering
	\caption{Periods comparison. The period of every MoveMine reference spot contrasted with that of the corresponding zone/s. The behavior  and its period, in the two cases, is reported in the last row. }
	\label{tab:comp}
	\begin{tabular}{C{2.5cm}|C{2.5cm}||C{3cm}|C{2.5cm}}
		\hline
		\textbf{Reference spot/behavior} & \textbf{Period (days)} & \textbf{Corresponding zone/behavior} & \textbf{Period (weeks) } \\
		\hline
		1 & 363 & $ \{\widehat{1},\widehat{2}\}$ & 51 \\
		2 & 363 & $ \widehat{3}$ & 53\\
		3 & 364 & $ \widehat{4}$ & 51\\
		1-2-3-2 & 363 & $\{\widehat{1}, \widehat{2}\}$-$\widehat{3}$-$\widehat{4}$ & 52\\
		\hline
	\end{tabular}
\end{table}
There is evidence that the \emph{reference spots} substantially match our zones. The only exception is the \emph{reference stop $\#$1} (New York area), which, in our model, covers both zone $\widehat{1}$ and  $\widehat{2}$. For homogeneity, as said above,  these two zones are considered as a unique region. 
The periods for each reference spot and corresponding zone/s, as well for the whole sequence (behavior), are summarized in Table \ref{tab:comp}. It can be seen that the results are coherent. 

%
To provide further details, the behavior of the eagle  is described in MoveMine as follows \citep{Movemine2011}: 
\emph{'This eagle stays in New York area (i.e., reference spot  $\#$1) from December to March. In March, it flies to Great Lakes area (i.e., reference spot  $\#$2) and
	stays there until the end of May. It flies to Quebec area (i.e., reference spot  $\#$3) in
	the summer and stays there until late September. Then it flies back to Great Lakes
	again staying there from mid October to mid November and goes back to New York in December'.} 
Interestingly, if we compare this behavior with our result, we can see that the sequence of reference stops is slightly different from the sequence of zones. In particular, in MoveMine,  the flight  from Quebec to the New York area, includes a stop at the area of Great Lakes, which our method seems not to recognize. Indeed, if we take a closer look, we can see that such a stop has a short duration ($<$ 20 days). 
Therefore, for how the parameters are chosen, SeqScan does not recognize such clusters as  stay regions. Note that such an observation has been made possible by the segmentation mechanism, which  discriminates between clusters and transitions, allowing for a detailed inspection of the behavior.

In summary, the two approaches appear substantially aligned. 
We emphasize, however, that there is a fundamental  difference between the two methods. MoveMine detects the reference spots as spatial-only clusters, and exploits signal processing techniques to extract from a noisy signal the sequence of temporally separated regions. Our technique does the opposite: it starts from the extraction of temporally separated regions, through the use of SeqScan, and finds  the zones. Consequently, the noise can be easily separated from the clusters at early stage, and that  simplifies the analysis. 

\section{Discussion}
In this work, we have used a research methodology that combines 
the investigation of a novel theoretical framework with an extensive validation of the technique. 
Actually, we have chosen to combine the two streams
to ensure a more robust evaluation  of the analytical framework, also in view of a possible deployment. Additional considerations:

\begin{itemize}
	\item 
	Validation strategies. 
	We have used different approaches to evaluate the effectiveness of SeqScan.
	Although not reported in this article for the sake of focus, we
	have contrasted SeqScan with two algorithms: the place detection algorithm proposed by Yu Zheng et al. \citep{zheng11} and ST-DBSCAN \citep{Birant2007}. 
	These two  techniques, however, rely on conceptual models of movement that are different from the one we refer to, therefore the comparison is unfair. Actually, the notion of local noise does not have a counterpart in any existing technique we are aware of. 
	Probably the most challenging question, with respect to validation,  is whether the proposed solution can be  effective in real applications and that motivates the concern for external validation practices \citep{Kriegel2010}. To that end, in \citep{Damiani2016} we have used a first approach where we evaluate SeqScan %
	using real animal trajectory data. The problem with real data is that if the  behavior  is inherently complex and only known  at macroscopic level (e.g. migratory behavior), domain experts may not be in the condition of classifying every point with sufficient confidence and thus the evaluation can be only conducted at a coarse level. In this sense, the use of a synthetic dataset built on an independent movement model conceptually encompassing the pattern of concern  has  dramatically improved the accuracy of the evaluation.
	
	\item Evaluation metrics. We have used  Purity and Pairwise F-measure. Yet,  these indexes  are specific for the evaluation of  traditional clustering while  the segmentation problem, we are dealing with, is somehow different. Indeed,  defining appropriate internal and external evaluation metrics for cluster-based segmentation is an open issue. A first proposal of internal indicator, called \emph{stationarity index} is presented in \citep{Damiani2016}. Applied to single clusters, the stationarity index is an estimate of the 'temporal density' in the cluster. 
	This topic will be investigated as part of future work.
	\item Generality of the proposed framework. As the external evaluation has been conducted on animal trajectories, one could raise the question on whether the scope of the solution is confined to the ecological domain. In reality, the model has been defined in a rigorous and general way, thus is prone to be applied in a variety of domains, such as human mobility analysis. 
	
\end{itemize}
The results of the evaluation process can be finally summarized as follows:
\begin{itemize}
	\item The experiments 
	show that overall the degree of matching of the SeqScan segmentation with the ground truth is  high (Tables 4-6). We recall that we have used the same set of parameters for all of the trajectories.  Therefore, it is likely that with a finer-grained tuning of the parameters, the  quality improves further.
	Importantly the ground truth is generated independently from the  clustering  while the evaluation has been conducted in a blind manner ignoring the simulation parameters. This is important for two reasons: it definitely supports the thesis that SeqScan can detect  this class of patterns; and that the evaluation is fair. 
	
	\item  For the practical application of  SeqScan, the generation of the function $f_T$, exemplified in Figure \ref{fig:graph_ind14_10days_x160}, can be extremely useful to determine a suitable set of values for $\delta$, in the same spirit of Optics \citep{Optics1999}.
	In addition the experiments show that SeqScan is  resilient to relatively low sampling rates.
	\item 
	Finally, a novel approach, grounded on the SeqScan framework, is proposed to support the discovery  of periodic locations and behaviors. The approach can compete with state-of-the-art techniques in detecting periodical behaviors with noise, while offering a flexible  and principled solution.   
	
\end{itemize}

\section{Conclusions}
To summarize, this article introduces the notion of clustering-based segmentation and presents an algorithm, SeqScan, that leverages the density-based paradigm to efficiently compute the segmentation of a trajectory based on  spatial density criteria. Moreover, the article presents an extensive evaluation of the solution, which includes the comparison  of the SeqScan clustering  with the ground truth. The resulting framework can be extended to support the discovery of additional patterns.
The trajectory dataset created as ground truth will be made publicly available.  
Additional information on the MigrO plug-in for the QGIS environment implementing the key functionalities for the SeqScan analysis is available at:
\hbox{http://mdamiani.di.unimi.it/}.

\fontsize{8pt}{12pt}\selectfont
\section*{} 
\textbf{Acknowledgments }We thank Walid Aref, Purdue University, for the discussion on the use of the WARP technique, and Roland Kays, NC Museum of Natural Sciences, for kindly providing the real data used in the experiments. 

\fontsize{10pt}{12pt}\selectfont





\section*{Appendix: the animal dataset}

\begin{longtable}{c c}
	\label{tab:table-appendix}
	\includegraphics[scale=0.3]{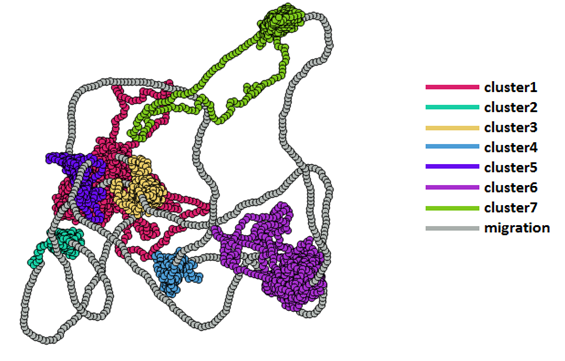} & \includegraphics[scale=0.3]{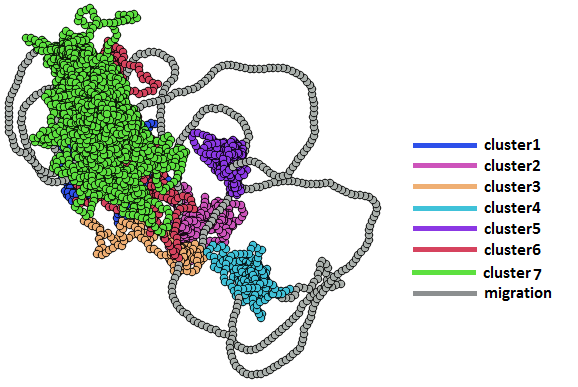}\\
	ind1 & ind6\\
	\includegraphics[scale=0.3]{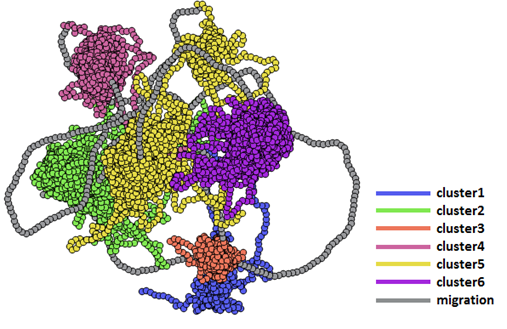} & \includegraphics[scale=0.3]{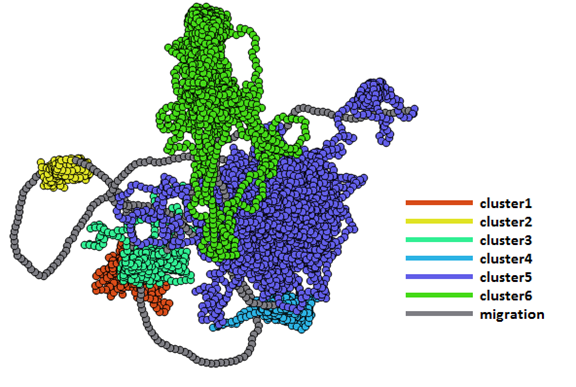}\\
	ind8 & ind10\\
	\includegraphics[scale=0.3]{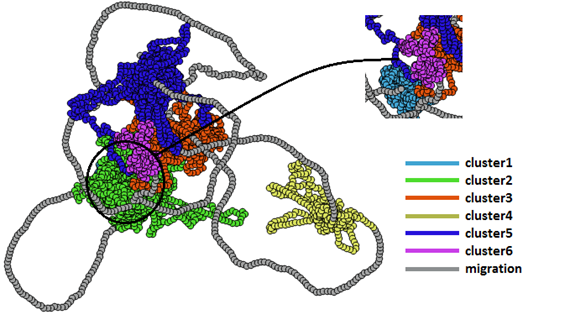} & \includegraphics[scale=0.3]{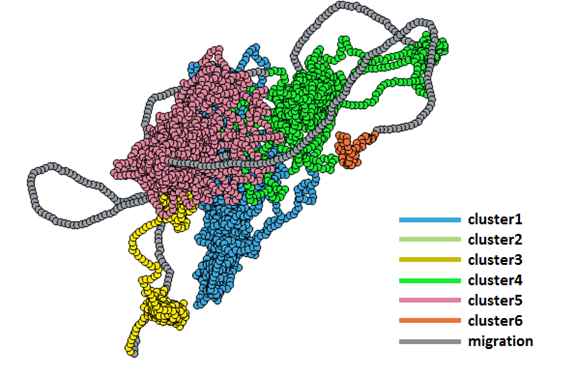}\\
	ind12 & ind14\\
	\includegraphics[scale=0.3]{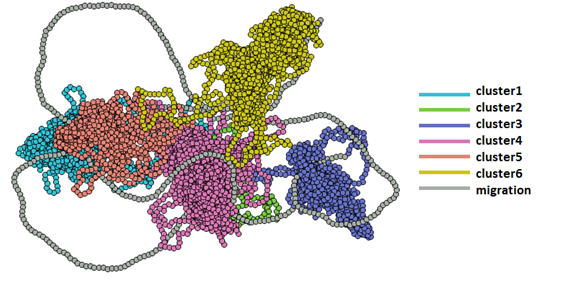} & \includegraphics[scale=0.3]{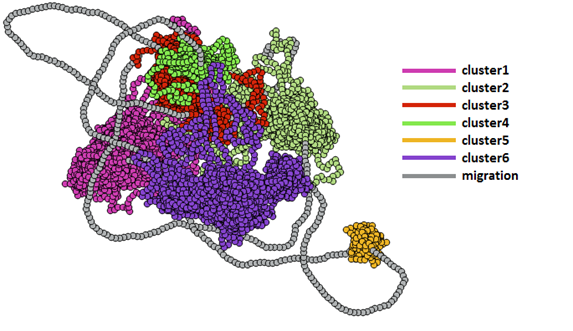}\\
	ind17 & ind25\\
	\includegraphics[scale=0.3]{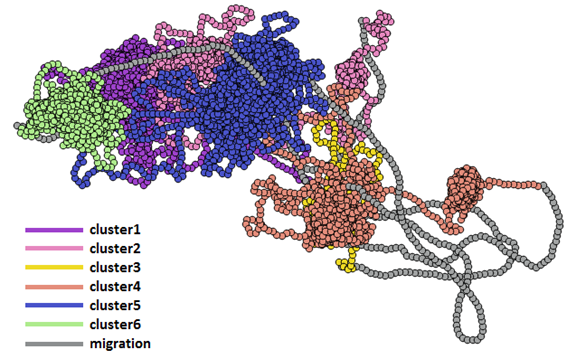} & \includegraphics[scale=0.3]{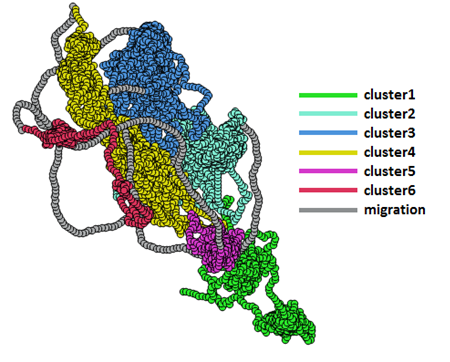}\\
	ind35 & ind39\\
	\includegraphics[scale=0.3]{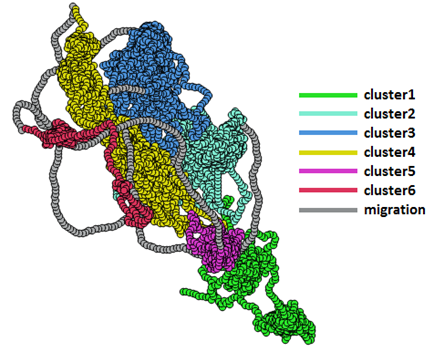} & \includegraphics[scale=0.3]{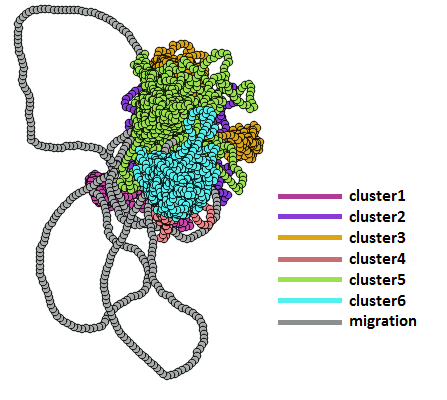}\\
	ind41 & ind49\\
	
\end{longtable}

\end{document}